\documentclass{article}

\PassOptionsToPackage{numbers,square}{natbib}

\usepackage[preprint]{neurips_2023}

\usepackage[utf8]{inputenc} % allow utf-8 input
\usepackage[T1]{fontenc}    % use 8-bit T1 fonts
\usepackage{hyperref}       % hyperlinks
\usepackage{url}            % simple URL typesetting
\usepackage{booktabs}       % professional-quality tables
\usepackage{amsfonts}       % blackboard math symbols
\usepackage{nicefrac}       % compact symbols for 1/2, etc.
\usepackage{microtype}      % microtypography
\usepackage{xcolor}         % colors

\usepackage{algorithm}
\usepackage{algorithmic}

\usepackage{amsmath}
\usepackage{graphics}
\usepackage{epsfig}
\usepackage{multirow}
\usepackage{xspace}
\usepackage{subfigure}
\usepackage{fancyhdr}
\usepackage{microtype}
\usepackage{color}
\usepackage{colortbl}
\usepackage{xcolor}
\definecolor{mygray}{gray}{.85}
\definecolor{myyellow}{RGB}{204,102,0}
\definecolor{myred}{RGB}{204,0,102}
\definecolor{mypurple}{RGB}{102,0,204}
\definecolor{maroon}{cmyk}{0,0.87,0.68,0.32}
\definecolor{myblue}{RGB}{227,227,240}

\usepackage{graphics}
\usepackage{amssymb}
\usepackage{amsthm}
\usepackage{mathrsfs}
\usepackage{booktabs}
\usepackage{algorithm}
\usepackage{algorithmic}
\usepackage{arydshln}
\usepackage{siunitx}
\usepackage{appendix}
\usepackage{enumitem}

\theoremstyle{plain}
\newtheorem{theorem}{Theorem}[section]

\theoremstyle{definition}

\newtheorem{assumption}[theorem]{Assumption}
\theoremstyle{remark}

\newcommand{\model}{GeoPro\xspace}

\title{Joint Design of Protein Sequence and Structure based on Motifs}

\author{%
  Zhenqiao Song$^{1}$, Yunlong Zhao$^{2,3}$, Yufei Song$^{1}$, Wenxian Shi$^{4}$, Yang Yang$^{2*}$, Lei Li$^{5}$\thanks{Corresponding author.}\ \\
  $^1$Department of Computer Science, University of California Santa Barbara \\
  $^2$Department of Chemistry and Biochemistry, University of California Santa Barbara\\
  $^3$Department of Chemistry, Massachusetts Institute of Technology \\ 
  $^4$Department of EECS, Massachusetts Institute of Technology \\
  $^5$Language Technology Institute, Carnegie Mellon University\\
  \texttt{\{zhenqiao,yufei\_song,yang89\}@ucsb.edu}, \texttt{leili@cs.cmu.edu
} \\
\texttt{\{yunlongz,wxsh\}@mit.edu
} \\
}

\begin{document}

\maketitle

\begin{abstract}
Designing novel proteins with desired functions is crucial in biology and chemistry.
However, most existing work focus on protein sequence design, leaving protein sequence and structure co-design underexplored.
In this paper, we propose \model, a method to design protein backbone structure and sequence jointly. 
Our motivation is that protein sequence and its backbone structure constrain each other, and thus joint design of both can not only avoid nonfolding and misfolding but also produce more diverse candidates with desired functions.
To this end, \model is powered by an equivariant encoder for three-dimensional~(3D) backbone structure and a protein sequence decoder guided by 3D geometry. 
Experimental results on two biologically significant metalloprotein datasets, including $\beta$-lactamases and myoglobins, show that our proposed \model outperforms several strong baselines on most metrics. 
Remarkably, our method discovers novel $\beta$-lactamases and myoglobins which are not present in protein data bank (PDB) and UniProt. These proteins exhibit stable folding and active site environments reminiscent of those of natural proteins, demonstrating their excellent potential to be biologically functional.
\end{abstract}

\section{Introduction}
A fundamental problem in protein engineering is designing novel proteins with desired biochemical functions such as catalytic activity \cite{fox2007improving}, therapeutic efficacy \cite{lagasse2017recent}, and fluorescence \cite{biswas2021low}.
Proteins embody their function through spontaneous folding of amino acid sequences into three dimensional~(3D) structures  \cite{go1983theoretical,chothia1984principles,starr2017exploring}.
In particular, protein's biochemical function is controlled by a subset of residues known as functional sites, or motifs \cite{wang2022scaffolding}.
Therefore, designing stably-folded proteins given a set of motifs is a promising direction to functional protein design.

In early representative work \cite{jiang2008novo,siegel2010computational}, manually prepared rules are applied to discover motifs.
Those rules are written by people with domain knowledge gained from the careful investigation of a specific protein family, which makes it hard to scale up to a wide number of protein families. 
Additionally, these efforts necessitate laborious trials and errors, making the overall process resource- and time-consuming.

\begin{figure}[htbp]
\begin{minipage}[t]{0.33\linewidth}
\centering
 \includegraphics[height=3.0cm]{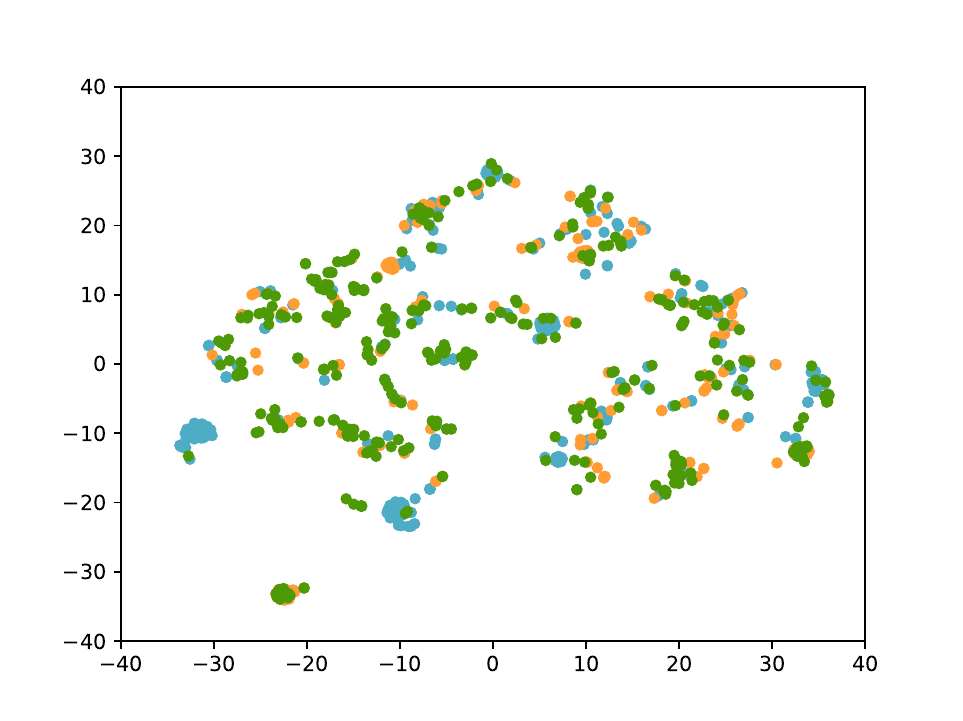}
\centerline{(a) ESM-2}
\end{minipage}%
\begin{minipage}[t]{0.33\linewidth}
\centering
\includegraphics[height=3.0cm]{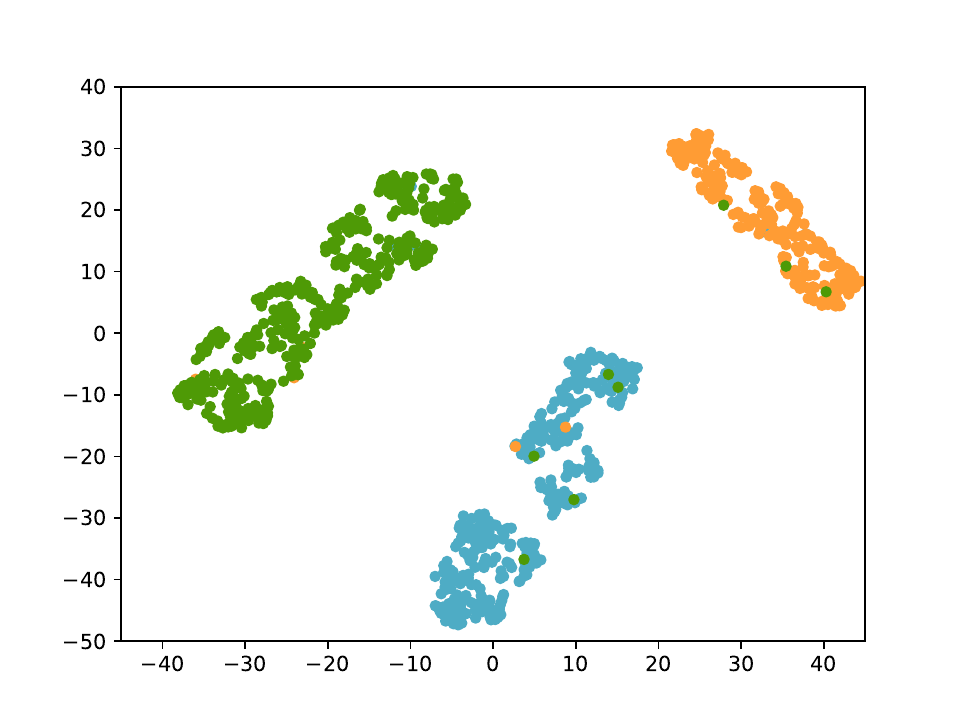}
\centerline{(b) Backbone Coordinate}
\end{minipage}%
\begin{minipage}[t]{0.33\linewidth}
\centering
\includegraphics[height=3.0cm]{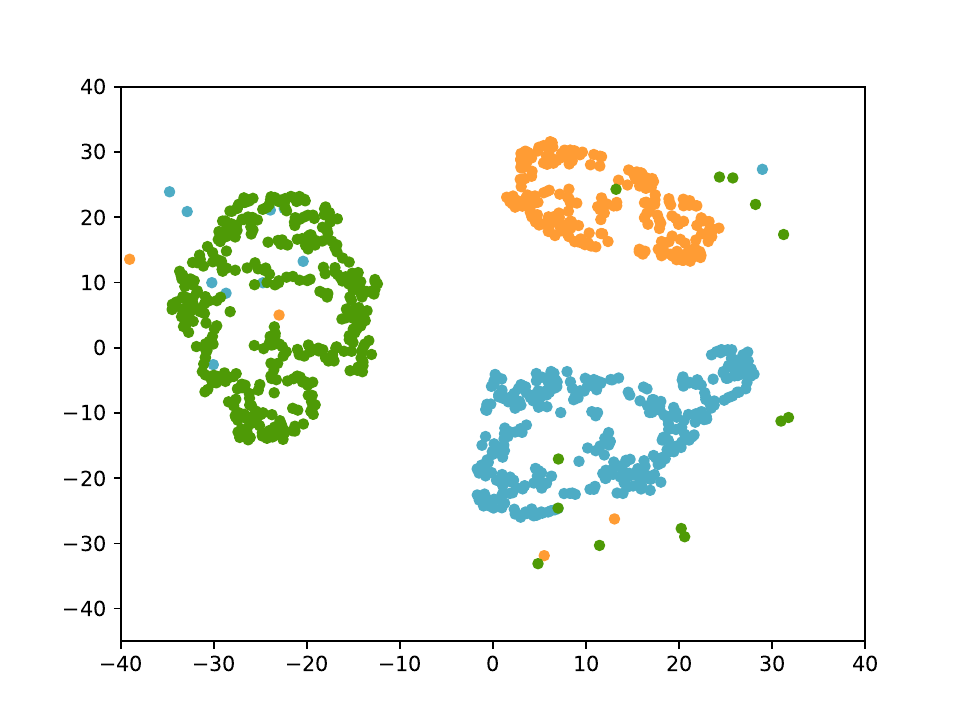}
\centerline{(c) Motif-Guided Decoder}
\end{minipage}
	\caption{Visualization of three proteins of myoglobin, each containing many instances obtained by different crystal methods. Figure~1(a) is the sequence representation from ESM-2~\cite{lin2022language}. Figure~1(b) shows the same proteins obtained by different crystal methods are closer to each other in $3$D space. Figure~1(c) demonstrates protein sequence representations with closer $3$D structure would also be clustered together after being revised by functional motifs.} 
 \label{figure_latent}
\end{figure}

To address these problems, machine learning methods have been employed to automate and improve the efficiency of protein design.
Recent efforts in this field primarily focus on either protein sequence design guided by fitness landscapes~\cite{brookes2019conditioning,gupta2019feedback, jain2022biological,ren2022proximal} or sequence generation from given 3D structures, also known as the inverse-folding problem~\cite{ingraham2019generative,hsu2022learning,dauparas2022robust,jing2020learning}.
However, the former approaches do not consider any structural information, oftentimes leading to unfolded or misfolded proteins, while the latter one depend on an entire protein structure, leading to limited candidate diversity and novelty. 
A more natural way is to co-design protein sequence and structure based on critical motifs.
However, current work either require  manually selected motifs by biochemical experts~\cite{anishchenko2021novo,wang2022scaffolding}, which is difficult to generalize to arbitrary proteins, or to extend design based on secondary structure topology~\cite{anand2022protein,shi2022protein}, which can not guarantee the designed proteins exhibit desired functions.
%, which guides the design process by the functional region while also allowing for flexibility outside the motifs.

% \citet{wang2021deep} propose inpainting method to recover missing regions based on the surrounding protein context. However, they need to prespecify an inpainting length between two contextual segments, which is unreasonable in reality since it is hard to decide without any prior knowledge.
% \citet{trippe2022diffusion} targets at creating stable scaffold supporting a target motif, and then generate a sequence based on the designed scaffold leveraging a well-trained ProteinMPNN~\cite{dauparas2022robust}. The issue of this pipeline is the sequence and structure can not benefit from each other during the design process.  
% Therefore, how to design desirable and novel proteins through the interaction between sequence and structure is still an under-studied and challenging task.

In this paper, we propose \model, an approach to design protein sequence and backbone structure jointly. 
Our motivation is that protein sequence and its backbone structure constrain each other, and thus joint design of both can not only avoid nonfolding and misfolding but also produce more diverse candidates with desired functions. 
\model first uses a fine-tuned ESM-2~\cite{lin2022language} to encode an initial protein sequence with only motifs into contextual representation. 
We design an equivariant graph neural network (EGNN) to encode and predict 3D backbone structure. 
This backbone structure encoder will  refine protein residue representations through their neighboring interactions in 3D space. 
We further design a sequence decoder to generate the full protein sequence given the  EGNN representations. 
\model predicts the backbone structure while incorporating contextual residue representations, and inpaints the protein sequence based on functional motifs. 
Our approach is advantageous because the mutual constraints between backbone structure and sequence facilitates the design of  stably-folded functional proteins.
We provide a theoretical proof that structures within a bounded variance lead to sequences that belong to the same protein, as illustrated in Figure~\ref{figure_latent}. This finding verifies that our discovered novel proteins are highly likely to be biologically functional.

We carry out extensive experiments on two metalloproteins, including $\beta$-lactamase and myoglobin, and compare the proposed model with several strong baselines. The contribution of this paper are listed as follows: 
\begin{itemize}[nosep,leftmargin=2.6em]
    \item We propose \model to co-design protein sequence and backbone structure. It is able to  design diverse, novel, and functional proteins that are not recorded in PDB and  UniProt~\footnote{\scriptsize{UniProt is a huge protein sequence database.}}.
    \item Experiments show that \model achieves highest performance on most metrics. The designed $\beta$-lactamases and myoglobins exhibit stable folding and can respectively bind their metallocofactors including zinc and heme, validating their excellent potential to be functional proteins.
\end{itemize}

\section{Related Work}
\textbf{Generative Protein Design}
Protein sequence design has been studied with a wide variety of methods, including traditional directed evolution \cite{arnold1998design,dalby2011strategy,packer2015methods,arnold2018directed} and machine learning methods~\cite{belanger2019biological,angermueller2019model,moss2020boss,terayama2021black}.
Following the success of deep generative models, there are some work focusing on protein sequence design with specific functions, aka. fitness.
They either search satisfactory sequences using deep generative models \cite{brookes2018design,brookes2019conditioning,madani2020progen,kumar2020model,das2021accelerated,hoffman2022optimizing,melnyk2021benchmarking,anishchenko2021novo,ren2022proximal}, or directly generate protein sequences applying deep generative models~\cite{jain2022biological,song2023importance}.
Another class of methods focus on inverse-folding problem~\cite{fleishman2011rosettascripts,ingraham2019generative,xiong2020increasing,mcpartlon2022deep,hsu2022learning},which targets at producing a protein sequence that can fold into a given structure.
Both approaches lack consideration for $3$D structure design, resulting in constrained accuracy and novelty in the design outcomes.
% Neither approach considers designing $3$D structure, leading to limited design accuracy and novelty.

\textbf{3D Protein Design}
\citet{wang2022scaffolding} propose inpainting method to recover both missing protein sequence and structure based on given motif segments. However, they need to pre-specify a possible inpainting length range, which is hard even for a biochemical expert.
\citet{trippe2022diffusion} frame motif-scaffolding problem as a conditional sampling process in diffusion models.
However, they only consider designing novel structures, which may not fully utilize the inherent correlation between protein sequence and structure.
\citet{anand2022protein} first propose to co-design protein sequence and structure conditioning on given secondary structures~(SS). Following their work, \citet{shi2022protein} propose to realize design conditioning on SS and contact map. However, knowing the topology of a protein before design process is difficult and also cannot guarantee the designed proteins have the desired functions which are mostly determined by side-chain of residues at functional sites.

In this paper, we focus on joint design of protein sequence and structure conditioning on critical motifs, which leverages the relationship between protein sequence and backbone structure to help design not only stably-folded but also diverse and novel proteins with desired functions.

\section{Background}
\subsection{Equivariance and Invariance}
\textbf{Equivariant Function}
A function $f$ is said to be equivariant to the action of a group $\mathcal{G}$ if $T_g(f(x))$ = $f(S_g(x))$ for all $g \in \mathcal{G}$, where $S_g$, $T_g$ are linear representations related to the group element~\cite{serre1977linear}. 
In this work, we consider the Euclidean group $E(3)$ generated by translations, rotations and reflections, for which $S_g$ and $T_g$ can be represented by a translation $t$ and an orthogonal matrix $R$ that rotates or reflects coordinates. $f$ is then equivariant to a translation $t$, rotation or reflection $R$ if transforming its input results in an equivalent transformation of its output, i.e. $f(Rx+t)=Rf(x)+t$.

\textbf{Invariant Distribution}
In our setting, a conditional distribution $p(y|x)$ is invariant to the action of rotation or reflection $R$ when:
\begin{equation}
\small 
p(y|x)=p(Ry|Rx)\; \text{or} \; p(y|x)=p(y|Rx)\qquad  \text{for all orthogonal matrix $R$} 
\end{equation}
% A conditional distribution is also invariant to $R$ transformations if:
% \begin{equation}
% \small 
% p(y|x)=p(y|Rx)\qquad  \text{for all orthogonal matrix $R$} 
% \end{equation} 
% \citet{xu2022geodiff} proved that if x $\sim$ p(x) is invariant to a group and the transition probabilities of a Markov chain y $\sim$ p(y|x) are equivariant, then the marginal distribution of y is invariant to group transformations as well.

\subsection{E(n) Equivariant Graph Neural Networks (EGNNs)}
EGNNs~\cite{satorras2021n} are a type of Graph Neural Network that satisfies the equivariance constraint. In this work, we consider interactions between all C$_\alpha$ in the backbone structure, and therefore assume a fully connected graph $\mathcal{G}$ with nodes $v_i \in V$.
Each node $v_i$ is endowed with coordinates $x_i \in R^3$ as well as $d$-dimensional features $h_i \in R^d$. In this setting, EGNN consists of the composition of equivariant convolutional layers~(EGCL): $x^{(l+1)},h^{(l+1)} = \text{EGCL}[x^l, h^l]$, which are defined as:
\begin{equation}
\begin{split}
\small 
m_{ij}&=\phi_e(h_i^l, h_j^l, d_{ij}^2, a_{ij}), \quad h_i^{l+1} = \phi_h(h_i^l, \sum_{j\ne i}\Tilde{e}_{ij}m_{ij}) \\
x_i^{l+1} &= x_i^l + \sum_{j\ne i}\frac{x_i^l-x_j^l}{d_{ij}+1}\phi_x(h_i^l, h_j^l, d_{ij}^2, a_{ij})
\end{split}
\label{equation_egnn}
\end{equation}
where $l$ indexes the layer, and $d_{ij} = ||x_i^l-x^l_j||_2$ is the Euclidean distance between nodes $(v_i, v_j)$, and $a_{ij}$ are optional edge attributes. 
$\Tilde{e}_{ij}=\phi_{inf}(m_{ij})$ is the attention score. 
All learnable components $(\phi_e,\phi_h,\phi_x,\phi_{inf})$ are parametrized by fully connected neural networks.
An entire EGNN architecture is then composed of $L$ EGCL layers, which applies the following non-linear transformation: $\hat{x}, \hat{h} = \text{EGNN}[x^0, h^0]$. This transformation satisfies the required equivariant property:
\begin{equation}
\small 
R\hat{x} + t, \hat{h} = \text{EGNN}[Rx^0 + t, h^0]\qquad   \text{for all orthogonal matrix $R$ and $t\in R^3$}
\label{equation:EGNNs}
\end{equation}

\section{Methods}
\begin{figure*}
  \centering
  \includegraphics[width=13.5cm]{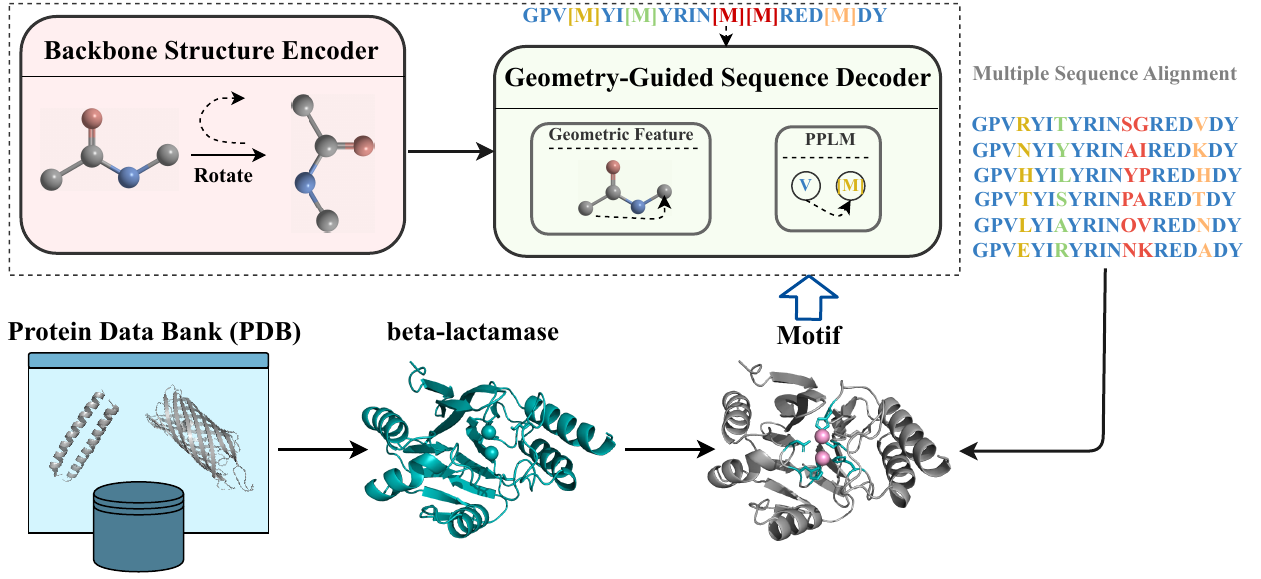}
  \vspace{-0.6em}
  \caption{The overall architecture of the proposed \model.}
  \label{Fig: model}
\end{figure*}

In this section, we describe our method in detail.
We first give the problem formulation in \ref{subsection4.1}.
Then we detail the interative process in which a backbone structure encoder predicts the backbone structure by incorporating the contextual residue representations in \ref{subsection4.2}, and a geometry-guided decoder inpaints the protein sequence based on the functional geometry-guided structure information in \ref{subsection4.3}.
Finally, we provide a theoretical proof in \ref{subsection4.4} that protein structures within a bounded variance lead to sequences which belong to the same protein. The overall architecture is illustrated in Figure~\ref{Fig: model}.
    
\subsection{Problem Formulation}
\label{subsection4.1}
Our goal is to design a protein with desired function by co-designing its $3$D backbone structure $x=\{x_1, x_2, ..., x_L\}$ and amino acid sequence $y=\{y_1, y_2, ..., y_L\}$ conditioning on a given motif $z=\{z_b, z_s\}$. 
L is the protein sequence length and $x_i \in R^3$ is the $3$D coordinate of the C$_\alpha$ of $i$-th amino acid $y_i \in \mathcal{V}$ -- $\mathcal{V}$ consists of 20 amino acids. 
A given motif z consists of backbone structure segments $z_b$ and sequence segments $z_s$.
To achieve this goal, we aim to maximize the joint probability of $x$ and $y$ conditioning on $z$:
\begin{equation}
\small 
p(x,y|z)=p(x|z; \theta)p(y|z; \phi)
\end{equation}
where $p(x|z;\theta)$ is modeled by the backbone structure encoder with parameter $\theta$, and $p(y|z; \phi)$ is modeled by the geometry-guided sequence decoder with parameter $\phi$.
% where $\theta$ and $\phi$ are the backbone structure encoder and geometry-guided sequence decoder parameters respectively. 
We will prove the joint probability $p(x,y|z)$ is invariant to the translation and rotation actions on $z$ in the following subsections.  

\subsection{Backbone Structure Encoder}
\label{subsection4.2}
The backbone structure encoder utilizes contextual residue representations to predict the protein backbone structure, while also refining the node features through their interactions with neighboring residues in $3$D space.

Specifically, we leverage EGNN (Equation~\ref{equation_egnn} and ~\ref{equation:EGNNs}) as the backbone structure encoder:
\begin{equation}
\small 
\hat{x}, \hat{h} = \text{EGNN}[x^0, h^0|; \theta]
\label{equation: whole_egnn}
\end{equation}
where $x^0=z_b\cup \hat{x}^0_{-z_b}$ is the initialized $3$D backbone coordinates, in which the motif $z_b$ keeps fixed while other nodes $\hat{x}^0_{-z_b}$ are sequentially initialized on a spherical surface to their neighboring amino acid. 
In particular, we find the Euclidean distance between every neighboring C$_\alpha$ is almost the same~(around $r=3.75$\AA).
Therefore, for a node $z_{k+1} \in \hat{x}^0_{-z_b}$, we initialize  it as a spherical surface with the center equals to its nearest neighbor $z_k \in z_b$:
\begin{equation}
\small 
\hat{x}^0_{k+1} = [z_k[0]+r\cdot\sin{\omega_1}\cos{\omega_2}, z_k[1]+r\cdot\sin{\omega_1}\sin{\omega_2}, z_k[2]+r\cdot\cos{\omega_1}]
\end{equation}
where $\omega_1 \sim \text{Uniform}(0, \pi)$ and $\omega_2 \sim \text{Uniform}(0, 2\pi)$. 
For $h^0$, we directly apply a pretrained protein language model~(PPLM)~\cite{lin2022language} to encode the corrupted protein sequence $\Tilde{y}=z_s\cup \Tilde{y}_{-z_s}$ with $\Tilde{y}_{-z_s}$ equals to [mask] symbol: $h^0=\text{PPLM}(\Tilde{y})$.
We initialize node features in this way because we believe a PPLM should provide some contextual information based on the large-scale protein sequences it has seen in the pretraining stage.

After getting the revised backbone structure, we minimize the Euclidean distance to learn the model:
\begin{equation}
\small
\mathcal{L}_b = \sum_{x_j \in x, x_j \not\in z_b} ||x_j -\hat{x}_j||_2^2
\end{equation}
which equals to maximize a three-dimensional Gaussian distribution whose mean vector equals to $x_j$ and covariance matrix equals to identity matrix for each node $x_j$:
\begin{equation}
\small 
\begin{split}
\log p(x|z;\theta)&=\log \Pi_{x_j \in x, x_j \not\in x_b} p(x_j|z;\theta) = \sum_{x_j \in x, x_j \not\in x_b} \log p(x_j|z;\theta)\\
\log p(x_j|z;\theta)&=\log \{\frac{1}{\sqrt{(2\pi)^3}}\exp{(-\frac{1}{2}(x_j -\hat{x}_j)^T(x_j -\hat{x}_j))}\} = -||x_j -\hat{x}_j||_2^2 + \text{const}
\end{split}
\end{equation}
As shown in Equation~\ref{equation:EGNNs}, EGNN satisfies the equivariance constraint, so we have:
\begin{equation}
\small 
\begin{split}
p(Rx_j+t|Rz+t;\theta)&=\frac{1}{\sqrt{(2\pi)^3}}\exp{(-\frac{1}{2}(Rx_j+t -R\hat{x}_j-t)^T(Rx_j+t -R\hat{x}_j-t))} \\
&=\frac{1}{\sqrt{(2\pi)^3}}\exp{(-\frac{1}{2}(x_j-\hat{x}_j)^TR^TR(x_j -\hat{x}_j))} 
=p(x_j|z;\theta)
\end{split}
\label{equa: back_invariance}
\end{equation}
where $R$ is an orthogonal matrix. 
Therefore, we can see the distribution of predicting backbone structure is invariant to the rotation or translation actions on motif.

The backbone structure encoder enables us to not only predict the backbone structure but also enhance the residue representations by considering their interactions with neighboring residues in $3$D space. This refinement process proves advantageous for the subsequent sequence inpainting procedure.

\subsection{Geometry-Guided Sequence Decoder}
\label{subsection4.3}
The geometry-guided sequence decoder~(GSD) reconstructs the protein sequence by utilizing functional-geometry structure information, thereby facilitating the discovery of novel proteins that exhibit similar biological functions to natural ones.

In particular, we leverage the residue features $\hat{h}$ revised by motif geometry to reconstruct the original protein sequence as follows:
\begin{equation}
\small 
p(y|z; \phi) =\prod_{y_j \in y, y_j \not\in z_s} p(y_j|z;\phi)= \text{GSD}(f(\hat{h}); \phi),\qquad 
f(\hat{h}_j) = \left\{
\begin{aligned}
\hat{h}_j &,  & {y_j \in z_s,} \\
\text{Emb[mask]} &, & {\text{otherwise}}
\end{aligned}
\right.
\end{equation}
where $\hat{h}_j$ is the output of backbone structure encoder of the $j$-th token in Equation~\ref{equation: whole_egnn}.
Here we use a finetuned ESM-2~\cite{lin2022language} to initialize GSD. Then to optimize GSD, we minimize the negative log likelihood of the none-motif parts:
\begin{equation}
\small 
\mathcal{L}_s = \sum_{y_j \in y, y_j \not\in z_s}-\log p(y_j|z;\phi) 
\end{equation}

As shown in Equation~\ref{equation:EGNNs}, $\hat{h}$ is invariant to the rotation or translation actions on $x^0$, and thus we can verify that $p(y|z; \phi)$ is invariant:
\begin{equation}
\small
p(y|Rz+t; \phi)=\text{GSD}(f(\hat{h});\phi)=p(y|z;\phi)
\label{equa: seq_invariance}
\end{equation}
Combining the conclusions in Equation~\ref{equa: back_invariance} and \ref{equa: seq_invariance}, we can prove the joint probability $p(x,y|z)$ is invariant to the translation and rotation actions on z:
\begin{equation}
\small 
p(Rx+t, y|Rz+t)=p(Rx+t|Rz+t; \theta)p(y|Rz+t;\phi)=p(x|z; \theta)p(y|z; \phi)=p(x, y|z)
\end{equation}
Accordingly, our overall training objective is defined as:
\begin{equation}
\small
\mathcal{L} = \alpha*\mathcal{L}_b + \beta * \mathcal{L}_s
\label{equation_training}
\end{equation}
where $\alpha$ and $\beta$ respectively control the importance of these two terms.
Jointly minimizing objective~\ref{equation_training} equals to maximize the joint probability $p(x, y|z)$.

\subsection{Theoretical Analysis}
\label{subsection4.4}
For each protein in PDB, there might be several instances due to different crystal methods, which have the same  or different sequences as well as slightly different coordinates.
Sequence representations incorporating this kind of geometric information would get much closer to each other than other different proteins as shown in Figure~\ref{figure_latent}(c). 
We can prove that our functional geometry-guided sequence decoding loss has an upper bound. It means proteins having the similar structures are mapped to the geometry-guided representations in a manner that are well-separated from others.
This theoretically verifies that our discovered novel proteins are highly possible to be realistic ones due to their similar active site environments to natural proteins.

In particular, we can regard our sequence reconstruction process as an auto-encoding process~\cite{kingma2013auto} with some perturbation $\mathcal{C}$: $\Tilde{y}=\mathcal{C}(y)$.
Here the encoding process is $g=E(y)=\text{EGNN}(\text{PPLM}(\mathcal{C}(y)))$, and decoding performs $y=\text{GSD}(g)$.
Let $\sigma$ denote the sigmoid function.

\begin{assumption}
The decoder D can approximate arbitrary $p(y|g)$ so long as it remains sufficiently Lipschitz continuous in g. Namely, there exists L > 0 such that decoder D obtainable via training satisfies: ${\forall} y \in \mathcal{Y}, {\forall} g^1, g^2 \in \mathcal{Z}$, $|\log p(y|g^1) -\log p(y|g^2)|\leq L||g^1-g^2||$. (We denote this set of decoders $\mathcal{D}_L$.)
\end{assumption}

\begin{theorem}
\label{theorem_1}
Suppose $\{y^1, y^2, ..., y^n\}$ are belonging to n/K different proteins of equal size K instances, with $s_i$ denoting the specific protein of $y^i$. Suppose $p(y^i|\Tilde{y}^j)=1/K$ if $s_i=s_j$ and 0 otherwise. With a deterministic encoder mapping E from $\{y^1, y^2, ..., y^n\}$ to $\{g^1, g^2, ..., g^n\}$, the denoising objective $\max_{D\in \mathcal{D}_L} \frac{1}{n}\sum_{i=1}^n \sum_{j=1}^n p(y^j|\Tilde{y}^i)\log p(y^i|E(y^j))$ has an upper bound: $\frac{1}{n^2} \sum_{i,j:s_i\ne s_j} \log \sigma(L||E(y^i)-E(y^j)||)-\log K$.
\end{theorem}
Note that $y^i$ here denotes the $i$-th protein sequence in the dataset. 
$\mathcal{Y}$ denotes the set of all protein sequences.
We provide the proof in appendix~\ref{theorem 3.2}.

\section{Experiments}
% In this section, we conduct extensive experiments to validate the effectiveness of our proposed method.

\subsection{Datasets}
We evaluate our method on two metalloproteins: $\beta$-lactamase binding zinc ion, and myoglobin binding heme.
We discuss the significance of both proteins in Appendix~\ref{reason_for_datasets}. 
To obtain the data, we first collect these two kinds of proteins from PDB and then extract chain A from these proteins. Only proteins capable of binding the corresponding metallofactors are retained.
Next we perform length filtering for both proteins, i.e., reserving $\beta$-lactamases longer than $200$ and myoglobins longer than $100$. 
Then we run MSAs using ClustalW2 method~\cite{anderson2011suitemsa}. Based on the results, we set positions whose alignment frequencies are higher than a given threshold $\lambda$ as the motif parts and others as flexible parts. 
Then we randomly split each dataset into training/validation/test sets with the ratio $8:1:1$.
The detailed data statistics are shown in Appendix Table~\ref{Tab: data_statistics}.

\subsection{Experimental Details}
\textbf{Training Details} We use two-layer EGNN~\cite{satorras2021n} with hidden size equal to $320$ as the backbone structure encoder and a finetuned ESM-2 \cite{lin2022language} to initialize the GSD.
The sequences are decoded using sampling strategy with top-$3$. 
More implementation details are provied in Appendix~\ref{appendix_implementation_details}.

\textbf{Baseline Models}
We compare the proposed \model against the following representative baselines:
(1) \textbf{Hallucination}~\citep{anishchenko2021novo} 
uses MCMC~\citep{andrieu2003introduction} incorporating a motif constraint into the acceptance score calculation.
(2) \textbf{Inpainting}~\citep{wang2022scaffolding} recovers both sequence and structure based on the given protein segments.
(3) \textbf{SMCDiff}~\citep{trippe2022diffusion}+\textbf{ProteinMPNN}~\citep{dauparas2022robust}: We first apply SMCDiff to design a protein structure based on given motifs and use ProteinMPNN to generate a sequence based on the given structure.

To better analyze the influence of different components in our model, we also conduct ablation tests:
(7) \textbf{\model-w/o-ctx}:The node feature of the backbone structure encoder is randomly initialized without any contextual residue information.
(9) \textbf{\model-w/o-geo} directly applies the residue embeddings from the finetuned ESM-2 instead of the functional geometry revised representations.
(9) \textbf{\model-ESM} directly applies the pretrained ESM-2 without any further training.

\textbf{Evaluation Metrics}
We use the following automatic metrics to evaluate the quality of the designed proteins:
(1) \textbf{AAR} assesses how similar the designed sequence is to the target sequence.
(2) \textbf{RMSD} evaluates how close our designed structure is to the target structure.
(3) \textbf{pLDDT}~\citep{jumper2021highly} provides an overall confidence score that a designed protein sequence can fold into a structure which is similar to natural proteins. We apply ESMFold~\citep{lin2023evolutionary} to calculate pLDDT due to its much higher efficiency than AlphaFold2~\citep{jumper2021highly}.
(4) \textbf{TM-score}~\citep{zhang2005tm} evaluates how similar structure predicted by the designed sequence is to the target structure. We use ESMFold to predict the structure of the designed sequence.
% We also provide the E-value rate~(EVR) to denote the proportion of the designed sequence whose E-value is lower than $1$e-$5$ among all the generated sequences in Appendix~\ref{appendix_cases}.

\begin{table*}[!t]
\small
\centering
\begin{tabular}{llccccc}
\midrule
& Models & AAR (\%,$\uparrow$) & RMSD (\AA,$\downarrow$) & pLDDT ($\uparrow$) & TM-score ($\uparrow$)\\
\midrule
\multirow{4}{*}{\rotatebox{90}{$\beta$-lactamase}}& Hallucination & $4.79$ & $--$ & $30.5511$ & $0.2918$ \\ 
&Inpainting & $16.73$ & $4.0599$ & $61.7679$ & $0.3790$  \\
&SMCDiff+PrteinMPNN & $19.94$ & $10.3960$ & $42.0375$ & $0.3458$  \\
&\cellcolor{myblue}\model & \cellcolor{myblue}$\textbf{\textbf{43.41}}$ & \cellcolor{myblue}$\textbf{2.9825}$ & \cellcolor{myblue}$\textbf{62.7349}$  &\cellcolor{myblue}$\textbf{0.4256}$  \\
\midrule
\multirow{4}{*}{\rotatebox{90}{myoglobin}} & Hallucination & $4.81$ & $--$ & $38.2817$ & $0.2754$ \\ 
&Inpainting & $39.59$ & $3.3751$ & $67.0813$ & $0.4391$\\
&SMCDiff+ProteinMPNN & $12.47$ & $8.0067$ & $34.5914$ & $0.2235$   \\
& \cellcolor{myblue}\model & \cellcolor{myblue}$\textbf{51.12}$ & \cellcolor{myblue}$\textbf{2.9891}$  & \cellcolor{myblue}$\textbf{77.3399}$ & \cellcolor{myblue}$\textbf{0.4656}$  \\
\bottomrule
\end{tabular}
\caption{Model performance on two metalloprotein datasets. \model achieves the best performance on both datasets.}
\label{Tab: main_results}
\end{table*}

\subsection{Main Results}
Table~\ref{Tab: main_results} reports the performance of all models.

\textbf{\model can design more realistic proteins.}
% plddt, tm-score, rmsd 
Table~\ref{Tab: main_results} shows our proposed \model achieves highest pLDDT and TM-score among all competitors on both datasets, demonstrating our model can generate more realistic proteins with relatively high confidence.
Our interpretation is that our \model leverages the inherent correlation between protein sequence and backbone structure, which can not only constrain each other to avoid misfolding but also benefit from each other to design a biologically functional protein with higher potential.

\begin{table*}[!t]
\small
\centering
\begin{tabular}{llcccc}
\midrule
& Models & AAR (\%,$\uparrow$) & RMSD (\AA,$\downarrow$) & pLDDT ($\uparrow$) & TM-score ($\uparrow$)   \\
\midrule
 & \cellcolor{myblue}\model & \cellcolor{myblue}$43.41$ & \cellcolor{myblue}$\textbf{2.9825}$ & \cellcolor{myblue}$\textbf{62.7349}$  &\cellcolor{myblue}$\textbf{0.4256}$\\
$\beta$-lactamase&~--  w/o-ctx & $\textbf{45.18}$ & $3.8759$ & $57.2319$ & $0.4125$  \\
&~--  w/o-geo & $42.07$ & $3.5755$ & $57.8732$ & $0.4109$\\
&~--  ESM & $39.98$ & $3.7218$ & $55.9130$ & $0.4019$ \\
\bottomrule
 & \cellcolor{myblue}\model & \cellcolor{myblue}$\textbf{51.12}$ & \cellcolor{myblue}$\textbf{2.9891}$  & \cellcolor{myblue}$\textbf{77.3399}$ & \cellcolor{myblue}$\textbf{0.4656}$ \\
myoglobin&~--  w/o-ctx & $46.19$ & $3.7615$ & $68.7901$ & $0.4502$ \\
&~--  w/o-geo & $49.56$ & $3.3132$ & $63.4686$ & $0.4419$  \\
&~--  ESM & $42.13$ & $4.0289$ & $61.6971$ & $0.4209$\\
\bottomrule
\end{tabular}
\caption{Ablation study results: Either removing sequence contextual or geometric guidance would lead to performance degradation, highlighting how co-design can enhance superior protein design.}
\label{Tab: Ablation}
\end{table*}

% \textbf{\model is able to design more novel proteins.}
% Our \model obtains lowest AAI scores except Inpainting on both datasets. 
% However, Inpainting achieves very low TM-scores~(< $0.5$), indicating it might produce unreliable proteins though they have lowest amino acid overlapping with existing sequences.
% Instead, proteins designed by our model have relatively low amino acid overlapping with Uniprot and also highest TM-scores. It demonstrates that our \model can generate both novel and realistic proteins due to the reason that our sequence decoder is able to explore more diverse proteins which are simultaneously constrained by the geometric information as we have proved in Theorem~\ref{theorem_1}.

\textbf{\model has the smallest recovery error.}
As shown in Table~\ref{Tab: main_results}, our model outperforms all competitors on  AAR and RMSD, exhibiting its best performance on both sequence and structure reconstruction.
It is because our \model takes advantage of the relationship between protein sequence and backbone structure. Co-designing these two parts leads to more accurate sequence and structure through their interaction with each other during the design process.

\begin{figure}
\begin{minipage}[t]{0.33\linewidth}
\centering
\includegraphics[width=3.5cm]{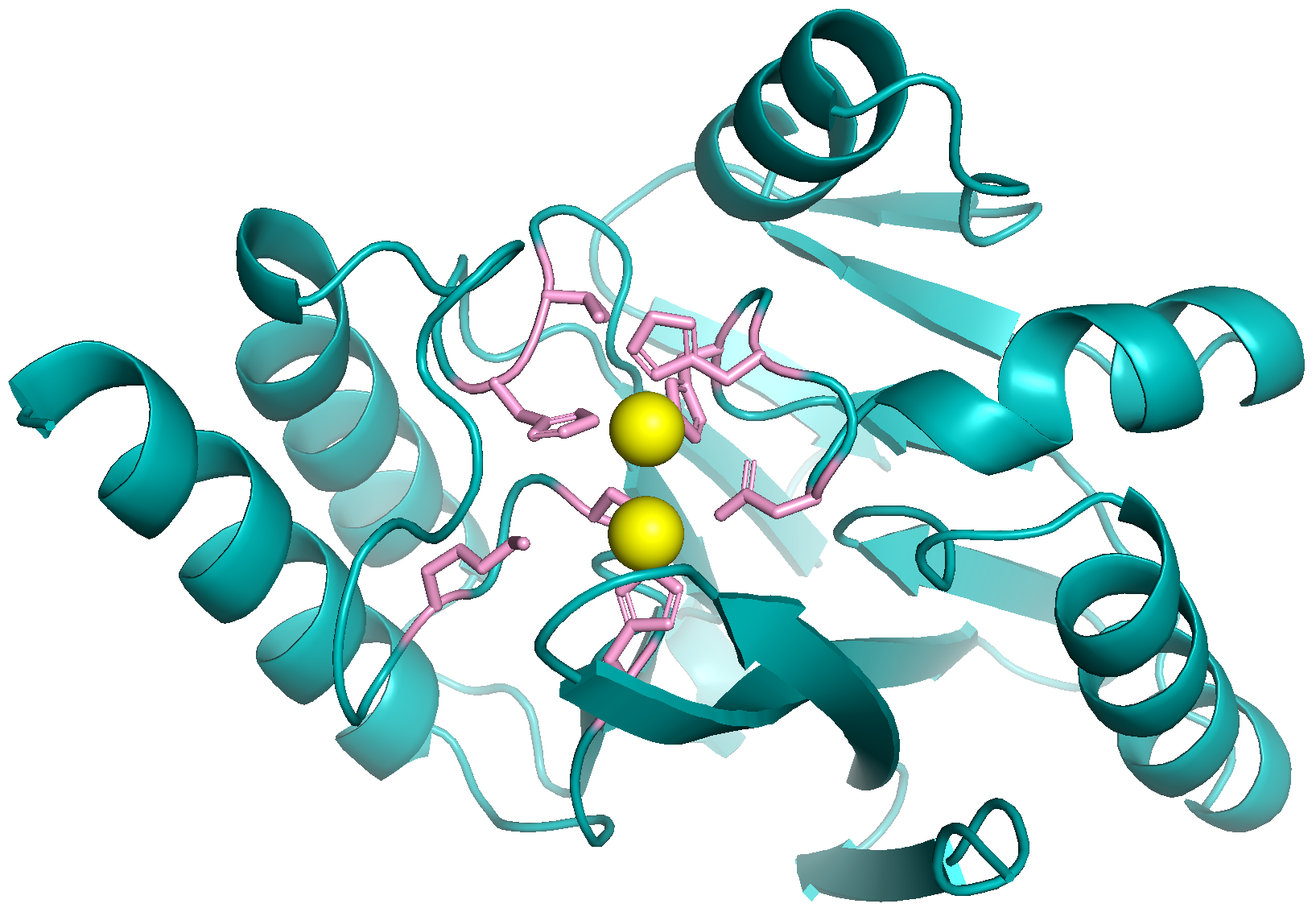}
\centerline{(a) case 1}
\end{minipage}%
\begin{minipage}[t]{0.33\linewidth}
\centering
\includegraphics[width=3.5cm]{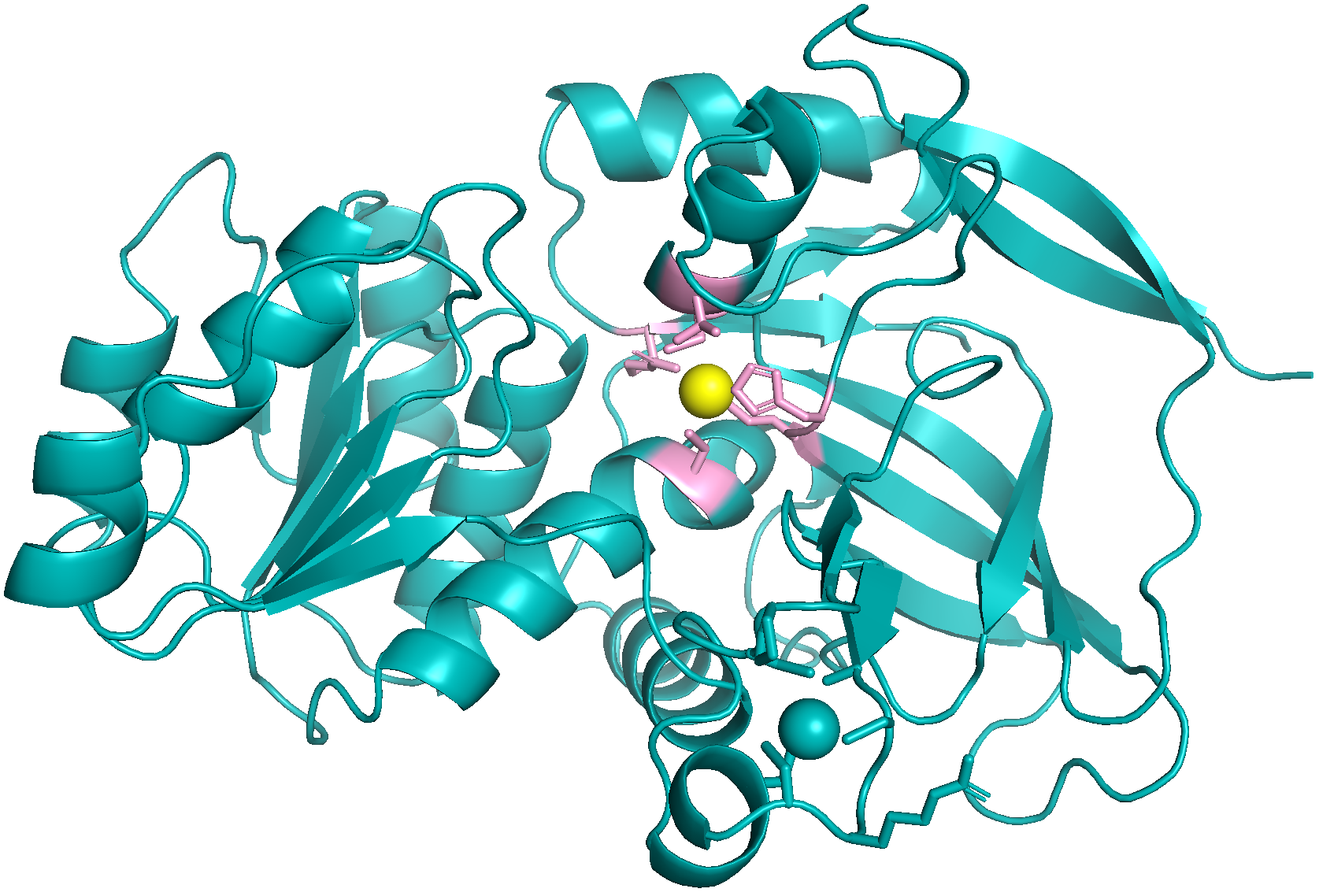}
\centerline{(b) case 2}
\end{minipage}%
\begin{minipage}[t]{0.33\linewidth}
\centering
\includegraphics[width=3.5cm]{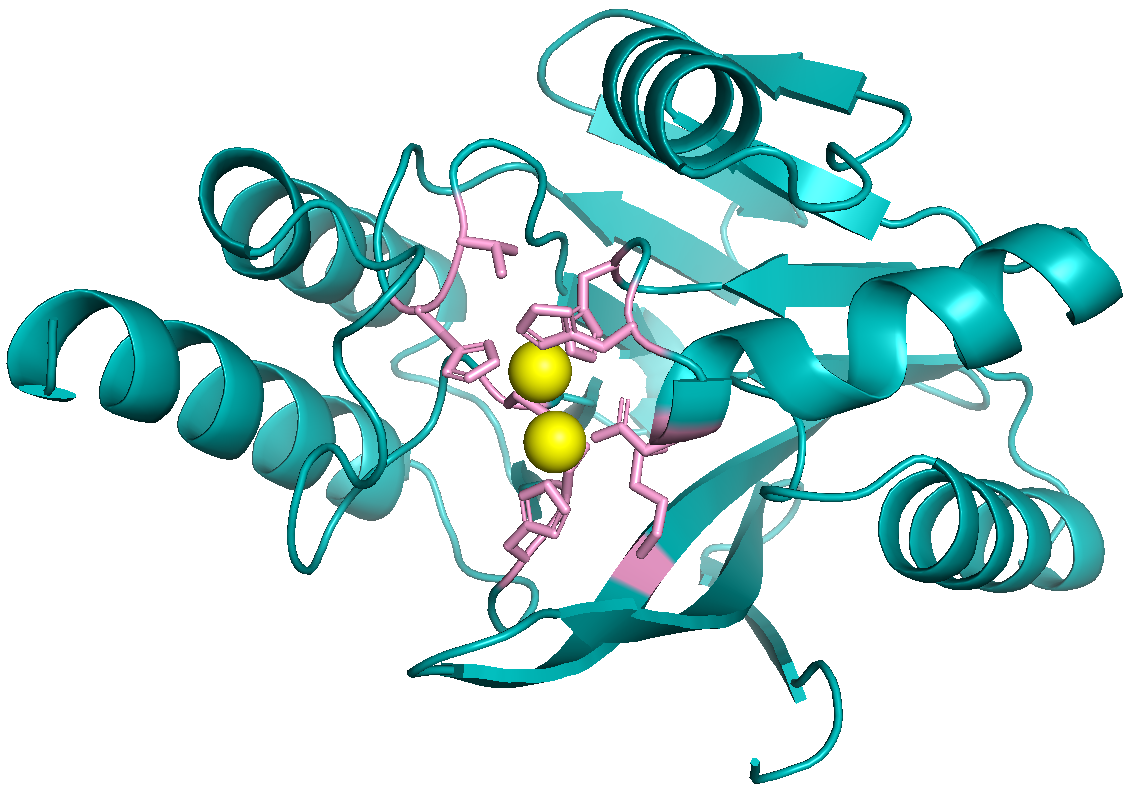}
\centerline{(c) case 3}
\vspace{-0.5em}
\end{minipage}
	\caption{Designed proteins of $\beta$-lactamase which belong to different subclasses (a) B1, (b) B2, (c) B3 metal-dependent $\beta$-lactamases.} 
 \label{figure_beta}
\end{figure}

\begin{figure}
\begin{minipage}[t]{0.33\linewidth}
\centering
\includegraphics[width=3.2cm]{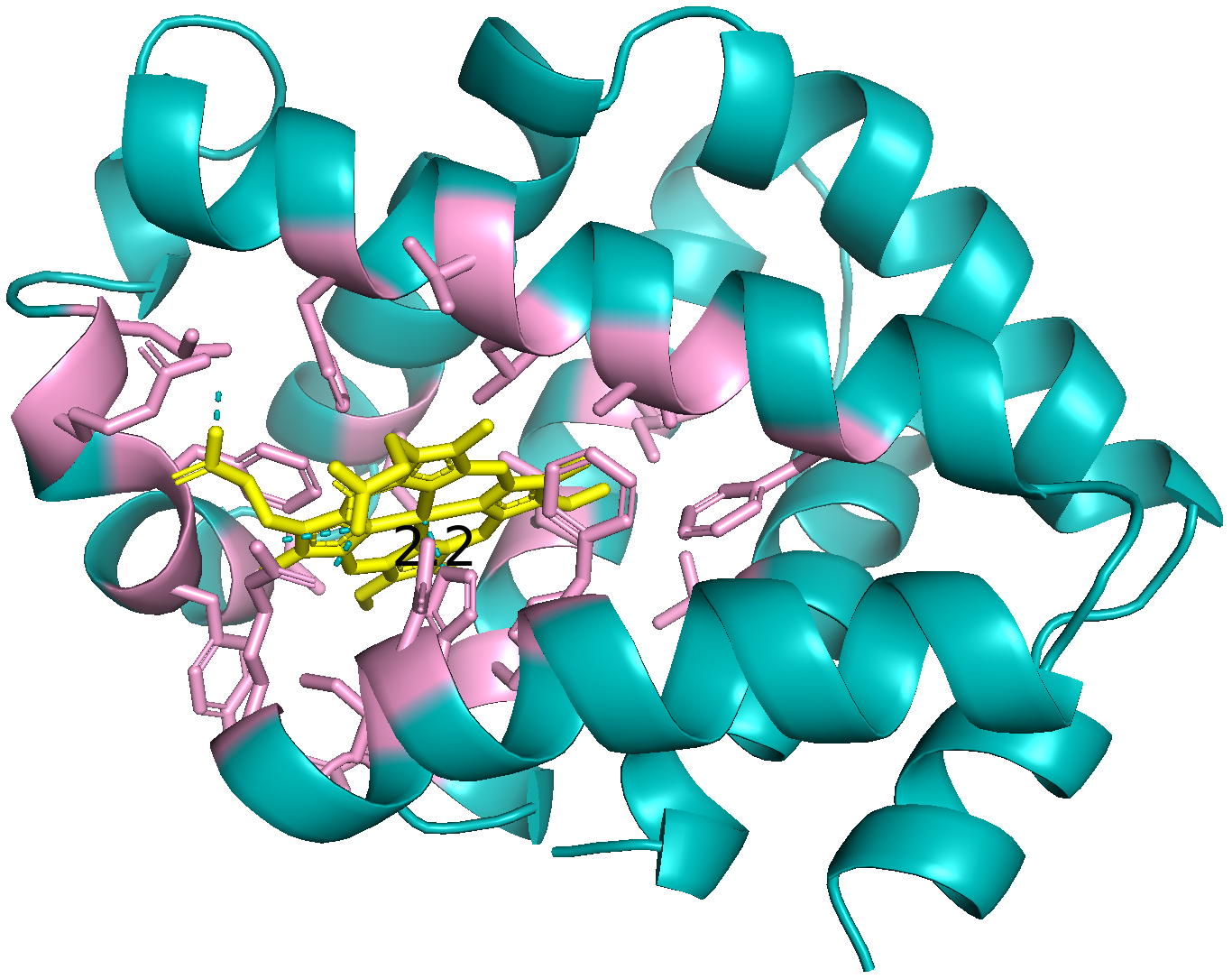}
\centerline{(a) case 1}
\end{minipage}%
\begin{minipage}[t]{0.33\linewidth}
\centering
\includegraphics[width=3.3cm]{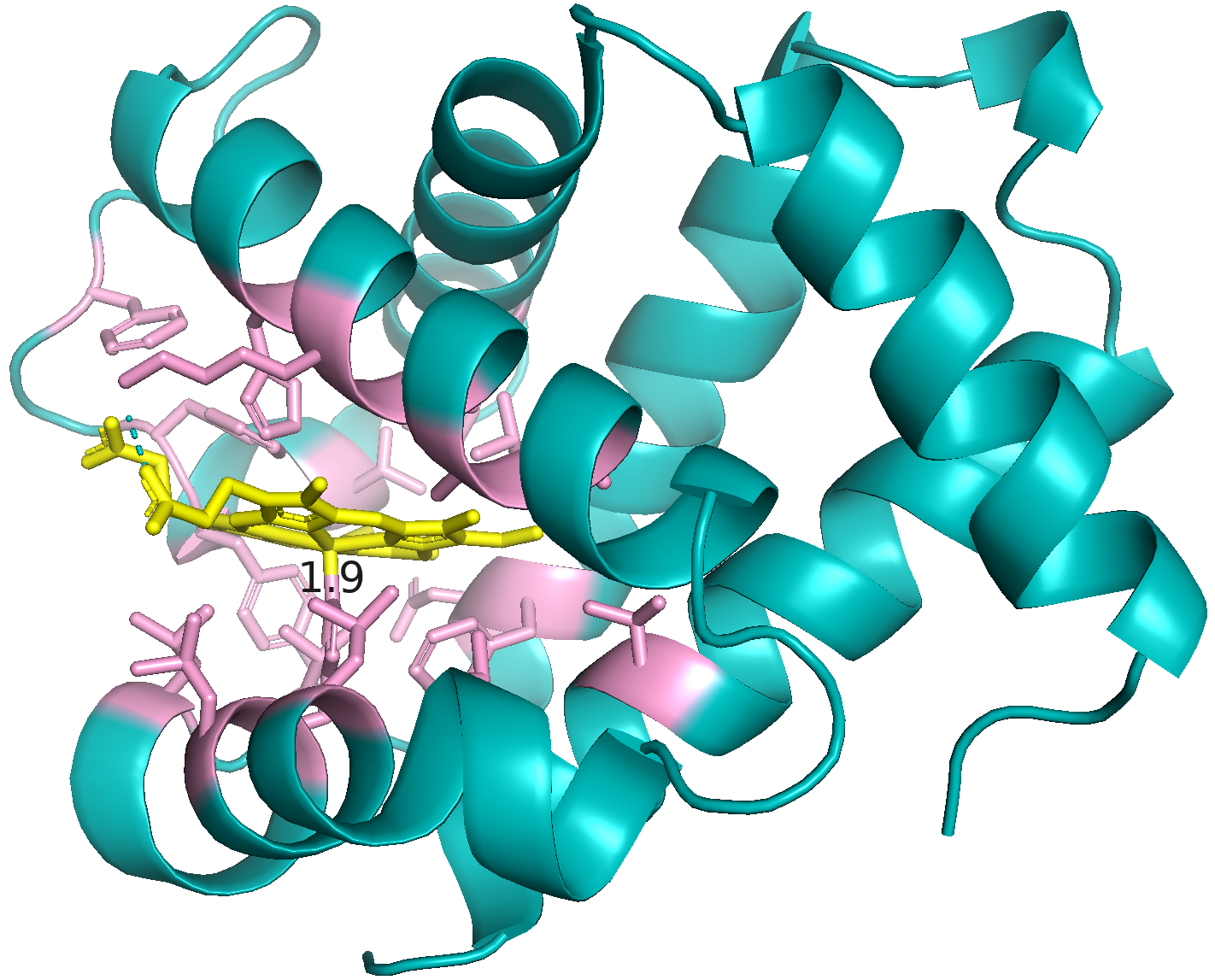}
\centerline{(b) case 2}
\end{minipage}%
\begin{minipage}[t]{0.33\linewidth}
\centering
\includegraphics[width=3.3cm]{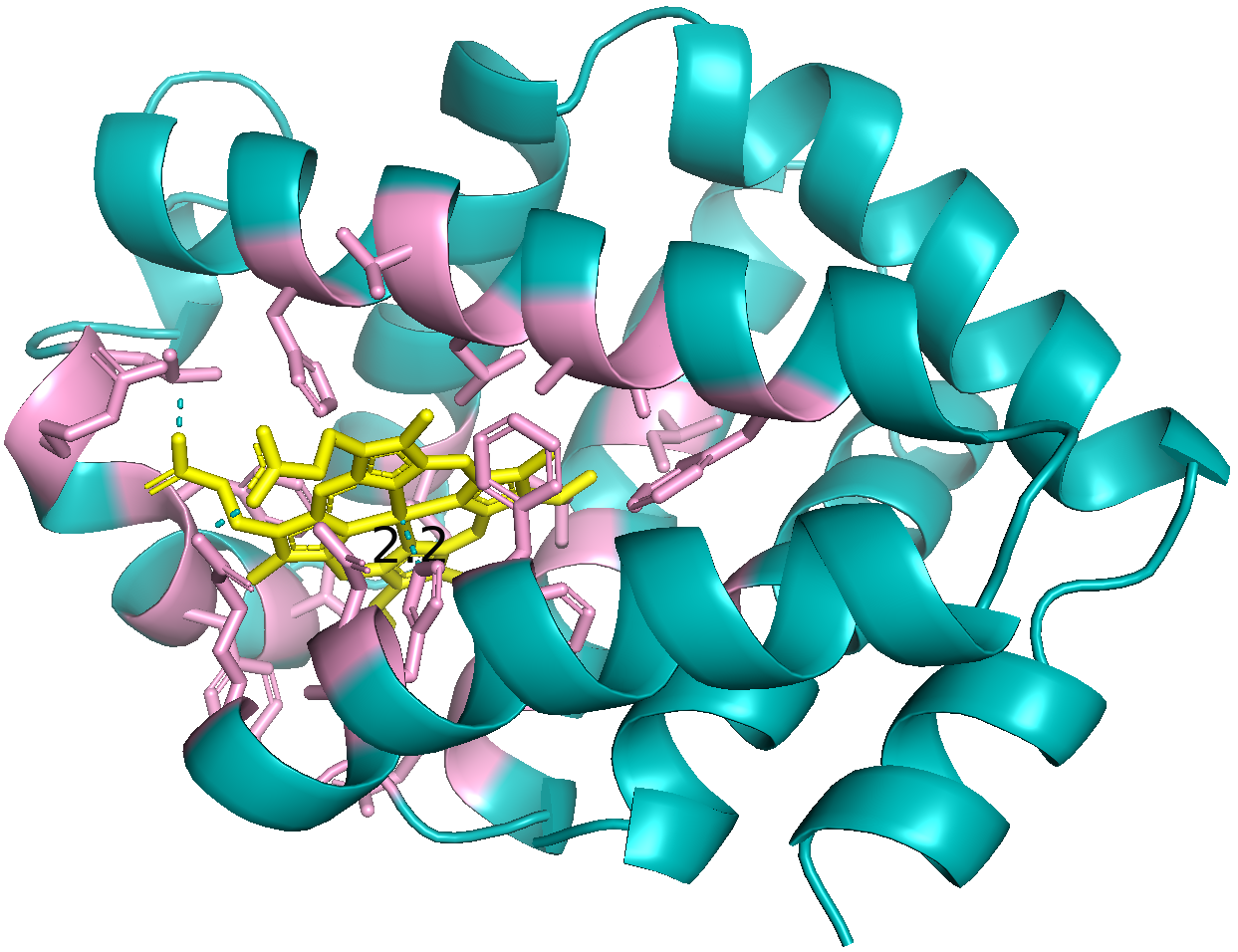}
\centerline{(c) case 3}
\end{minipage}
 \caption{Designed proteins of myoglobin which bind heme through hydrogen bonds.} 
 \label{figure_myoglobin}
\end{figure}

% \begin{figure}
% \begin{minipage}[t]{0.33\linewidth}
% \centering
% \includegraphics[width=4.5cm]{plddt_beta.pdf}
% \centerline{(a) pLDDT of beta-lactamase}
% \end{minipage}%
% \begin{minipage}[t]{0.33\linewidth}
% \centering
% \includegraphics[width=4.5cm]{plddt_myoglobin.pdf}
% \centerline{(b) pLDDT of myoglobin}
% \end{minipage}%
% \begin{minipage}[t]{0.33\linewidth}
% \centering
% \includegraphics[width=4.15cm]{motif_threshold.png}
% \centerline{(b) Effect of motif threshold}
% \end{minipage}
% \vspace{-0.5em}
% 	\caption{Visualization of (a) and (b) pLDDT difference between ESMFold and AlphaFold2, (c) AAR and pLDDT under different motif threshold. Shadow denotes results for sampling five times.} 
%  \label{figure_analysis}
% \end{figure}

\subsection{Ablation Study}
Table~\ref{Tab: Ablation} shows the results of ablation study.
Without finetuning the sequence decoder, the performance drops the most on all metrics~(\model-ESM). It validates that related information can be better  explored from a pretrained protein language model by tuning the parameter to the specific protein family.
Without the contextual sequence representation initialization, \model-w/o-ctx achieves worse RMSD scores than \model-w/o-geo. 
Instead, after removing the geometric information guidance, \model-w/o-geo obtains lower amino acid recovery rate than \model-w/o-ctx.
Once again, these phenomena validate the mutual benefits between protein sequence and structure, highlighting how co-design can enhance the generation of superior protein candidates.

\section{Analysis}

To intuitively know how well our model can perform, we visualize the designed proteins.  
Specifically, we first randomly select $3$ cases from the top-100 candidates according to pLDDT and then use AlphaFold2 to predict the protein structure, which will be subsequently used to predict the relevant ligand by AlphaFill~\cite{hekkelman2023alphafill}.
Finally we apply PyMOL~\cite{pymol} to visualize the final results.
From Figure~\ref{figure_beta} and~\ref{figure_myoglobin}, we can see that all the designed $\beta$-lactamases and myoglobins have active site environments highly similar to natural proteins, and can respectively bind zinc ion and heme (yellow parts), demonstrating our model can design functional proteins.
Furthermore, all these cases are not included by PDB, and some of them even have lower identity overlapping with UniProt sequences (e.g., 61.0\% AAR of Figure~\ref{figure_beta} (b)), verifying our \model is able to design novel proteins. 
% For example, Figure~\ref{figure_beta} (b) only has 57.8\% identity overlapping rate with the most similar sequence in UniProt.
Additionally, the three $\beta$-lactamases belong to three different subclasses of metal-dependent $\beta$-lactamases featuring complementary Zn coordination chemistries, validating that our model can design diverse proteins.
% Overall, these case studies show that our proposed \model has the capability to design novel and diverse proteins with desired functions.
We provide more cases in Appendix~\ref{appendix_cases} and all the cases are supplied in supplementary material.

\section{Conclusion}
This paper proposes \model, a method to co-design protein sequence and backbone structure.
The proposed model leverages the inherent correlation between protein sequence and backbone structure, and is powered by an equivariant $3$D backbone structure encoder and a geometry-guided sequence decoder.
Experimental results show that our proposed \model outperforms several strong baselines on most metrics. 
Additionally, our model discovers novel $\beta$-lactamases and myoglobins which are not recorded by PDB and UniProt.
One limitation of this work is, although our model has demonstrated promising results, the designed proteins have not undergone wet-lab testing. Consequently, we cannot provide complete assurance regarding the ability of the designed metalloproteins to bind the corresponding metallocofactor and perform their desired functions.
Future work will involve the wet-lab testing to verify the metal binding ability and biological function of our designed metalloproteins.
% Secondly, we aim to enhance our current methodology by incorporating motif learning as part of our model, enabling de novo protein design without constraints.
% context-free protein design and ultimately facilitating real de novo protein design.

\bibliography{neurips}
\bibliographystyle{plainnat}

%%%%%%%%%%%%%%%%%%%%%%%%%%%%%%%%%%%%%%%%%%%%%%%%%%%%%%%%%%%%
\newpage
\appendix
\section*{Appendix}
\section{Proof of Theorem 3.2}
\label{theorem 3.2}
\begin{theorem}
Suppose $\{y^1, y^2, ..., y^n\}$ are belonging to n/K different proteins of equal size K instances, with $s_i$ denoting the specific protein of $y^i$. Suppose $p(y^i|\Tilde{y}^j)=1/K$ if $s_i=s_j$ and 0 otherwise. With a deterministic encoder mapping E from $\{y^1, y^2, ..., y^n\}$ to $\{g^1, g^2, ..., g^n\}$, the denoising objective $\max_{D\in \mathcal{D}_L} \frac{1}{n}\sum_{i=1}^n \sum_{j=1}^n p(y^j|\Tilde{y}^i)\log p(y^i|E(y^j))$ has an upper bound: $\frac{1}{n^2} \sum_{i,j:s_i\ne s_j} \log \sigma(L||E(y^i)-E(y^j)||)-\log K$.
\end{theorem}

%\begin{proof}
\begin{proof}\renewcommand{\qedsymbol}{}
For convinience, we denote $g^j=E(y^j)$. We consider what is the optimal decoder probability assignment $p(y^i|g^j)$ under the Lipschitz constraint. Suppose $g^i, g^j$ satisfy that with some $0<\delta<\zeta$: $||g^i-g^j||<\delta$ if $s_i=s_j$ and $||g^i-g^j||>\zeta$ otherwise. 
To lower bound the training objective, we can choose:
\begin{equation}
\small 
p(y^i|g^j) = \left\{
\begin{aligned}
\frac{1-\gamma}{K} &,  & {s_i=s_j} \\
\frac{\gamma}{n-K} &, & {\text{otherwise}}
\end{aligned}
\right.
\end{equation}
with $\gamma=\sigma(-L\zeta) \in (0, \frac{1}{2})$, where $\sigma$ denotes sigmoid function. Note that this choice can ensure $\sum_{i\in [n]}p(y_i|g^j)=1$ for each $j\in [n]$, and this also does not violate Lipschits condition.

The objective is to find optimal decoder D which:
\begin{equation}
\begin{split}
&\max_{D\in \mathcal{D}_L} \frac{1}{n}\sum_{i=1}^n \sum_{j=1}^n p(y^j|\Tilde{y}^i)\log p(y^i|g^j) \\
=&\max_{D\in \mathcal{D}_L} \frac{1}{nK} \sum_j\sum_{i: s_i=s_j} \log p(y_i|g^j) 
\end{split}
\end{equation}
Now let us define $P_j=\sum_{i:s_i=s_j} p(y^i|g^j)=K\cdot p(y^j|g^j)$, the above optimizing objective becomes:
\begin{equation}
\small 
\begin{split}
&\max_{D\in \mathcal{D}_L} \frac{1}{nK} \sum_j\sum_{i: s_i=s_j} \log p(y_i|g^j)  \\
=&\max_{D\in \mathcal{D}_L} \frac{1}{n} \sum_j \log p(y_j|g^j) \\ 
=&\max_{D\in \mathcal{D}_L} \frac{1}{n} \sum_j \log P_j - \log K \\ 
=&\max_{D\in \mathcal{D}_L} \frac{1}{2n^2} \sum_i\sum_j (\log P_i + \log P_j) - \log K \\ 
\leq&\frac{1}{2n^2} \sum_i\sum_j \max_{D\in \mathcal{D}_L}(\log P_i + \log P_j) - \log K
\end{split}
\end{equation}
For each term $\max_{D\in \mathcal{D}_L}\log P_i + \log P_j$:
\begin{equation}
\small 
\begin{split}
\log P_i + \log P_j &= 2 \log K \cdot \frac{1-\gamma}{K} \\ 
&=2\log (1-\gamma) \\ 
&=2\log \frac{\exp(L\zeta)}{1+\exp(L\zeta)} \\ 
&=2\log \sigma(-L\zeta) \\ 
&\leq 2 \log \sigma(-L||g^i-g^j||)
\end{split}
\end{equation}
Overall, we have:
\begin{equation}
\small
\begin{split}
&\max_{D\in \mathcal{D}_L} \frac{1}{nK} \sum_j\sum_{i: s_i=s_j} \log p(y_i|g^j)  \\
\leq & \frac{1}{n^2} \sum_i\sum_{j:s_i\ne s_j}  \log \sigma(-L||g^i-g^j||) - \log K
\end{split}
\end{equation}
\end{proof}

% \begin{table}[!t]
% \small
% \begin{center}
% \begin{tabular}{lcc}
% \midrule
% Models & Beta-Lactamase & Myoglobin \\
% \midrule
% Structured Transformer &  & \\ 
% ProteinMPNN & $93.16$ & $98.49$ \\
% GVP& & \\
% \cdashline{1-3}[1pt/1pt]
%  Hallucination & $0.00$ & $0.00$\\
% Inpainting & $95.18$ & $80.00$ \\
%  SMCDiff & $--$ & $--$ \\
% \cdashline{1-3}[1pt/1pt]
% \cellcolor{myblue}\model &\cellcolor{myblue}$\textbf{94.51}$ & \cellcolor{myblue}91.52 \\
% ~--  RandGNN & $100.00$ & $86.77$ \\
% ~--  w/o-Geo & $97.48$ & $89.47$\\
% ~--  w/o-tune & $89.17$ & $85.97$\\
% \bottomrule
% \end{tabular}
% \end{center}
% \caption{EVR (\%,$\uparrow$) of all models. SMCDiff~\cite{trippe2022diffusion} focuses on scaffolding-motif problem without considering sequence, so we do not report its EVR.}
% \label{Tab: EVR}
% \end{table}

\section{Additional Experimental Details}

\subsection{Significance of Metalloproteins Studied Herein}
\label{reason_for_datasets}
Metalloproteins comprise almost 50\% of all the naturally occuring proteins. 
The Zn-dependent $\beta$-lactamases and Fe-dependent myoglobins studied herein represent biologically signicant metalloprotein examples.
In particular, $\beta$-lactamases are enzymes produced by microorganisms to break down $\beta$-lactam antibiotics, conferring antibiotic resistance~\cite{gupta2008metallo}. Thus, the study and design of $\beta$-lactamases hold relevance to public health and play a critical role in the development of new antibiotics.
On the other hand, myoglobin is a heme-containing protein involved in oxygen storage and transport in muscle tissue~\cite{springer1994mechanisms}, highlighting their biological significance.

% Myoglobin is the first protein to have its structure determined so it is where the science of protein structure really begins~\cite{ordway2004myoglobin}. Myoglobin is not necessarily more complex than other proteins, as all proteins have unique structures and functions that also make them interesting targets for study and design.
% For beta-lactmases, they are more complex than many other enzymes due to their ability to break down a wide range of beta-lactam antibiotics with different chemical structures.
% The complexity of beta-lactamases has made their design a significant challenge, but also a promising avenue for developing new therapeutics and combating antibiotic resistance.

Overall, these two proteins both have significant research values.
Our experimental results show our proposed \model has the capability to design novel proteins with desired functions for both general protein and enzyme, exhibiting its superior generalization.

\subsection{Protein Data Statistics}
Detailed data statistics for $\beta$-lactamase and myoglobin are reported in Table~\ref{Tab: data_statistics}.
We will release the two created metalloprotein datasets in the near future.

\begin{table}[!t]
\small
\begin{center}
\begin{tabular}{lccc}
\midrule
Protein  & PDB & Metal Binding & Length Filtering  \\
\midrule
$\beta$-lactamase & $171,484$ & $7,802$ & $5,427$\\
myoglobin & $14,573$ & $3,381$ & $3,381$ \\
\bottomrule
\end{tabular}
\end{center}
\caption{Detailed data statistics of the two metalloproteins.}
\label{Tab: data_statistics}
\end{table}

\subsection{More Implementation Details}
\label{appendix_implementation_details}
The mini-batch size and learning rate are set to $4$ sequences and $1$e-$7$ respectively.
The model is trained for $10$ epochs with $1$ NVIDIA RTX A$6000$ GPU card. 
% We apply Adam~\cite{kingma2014adam} as the optimizer.
We apply Adam~\cite{kingma2014adam} as the optimizer with a linear warm-up over the first $4,000$ steps and linear decay for later steps.
We tune the hyperparameters $\alpha$ and $\beta$ both from $0.1$ to $1.0$ with step size 0.1 and find $\alpha=0.1$ and $\beta=1.0$ performs best on the validation set for $\beta$-lactamase.
For myoglobin, we find the gradient tends to vanish when $\alpha$ is relatively large, and thus we tune it from $0.01$ to $0.1$ with step size $0.01$ and find $\alpha=0.01$ performs best on the corresponding validation set.

% \subsection{E-value Rate}
% To asses how close our designed sequences are to their homologs in Uniprot, we also report the E-value rate~(EVR), which represents the proportion of the designed sequence whose E-value is smaller than $1$e-$5$~\cite{yu2013construction} among all the designed sequences. The higher the EVR, the more designed sequences are possible to be real proteins.

\begin{figure}[htbp]
\centering
\subfigure[case 1]{
\begin{minipage}[t]{0.33\linewidth}
\centering
\includegraphics[width=3.5cm]{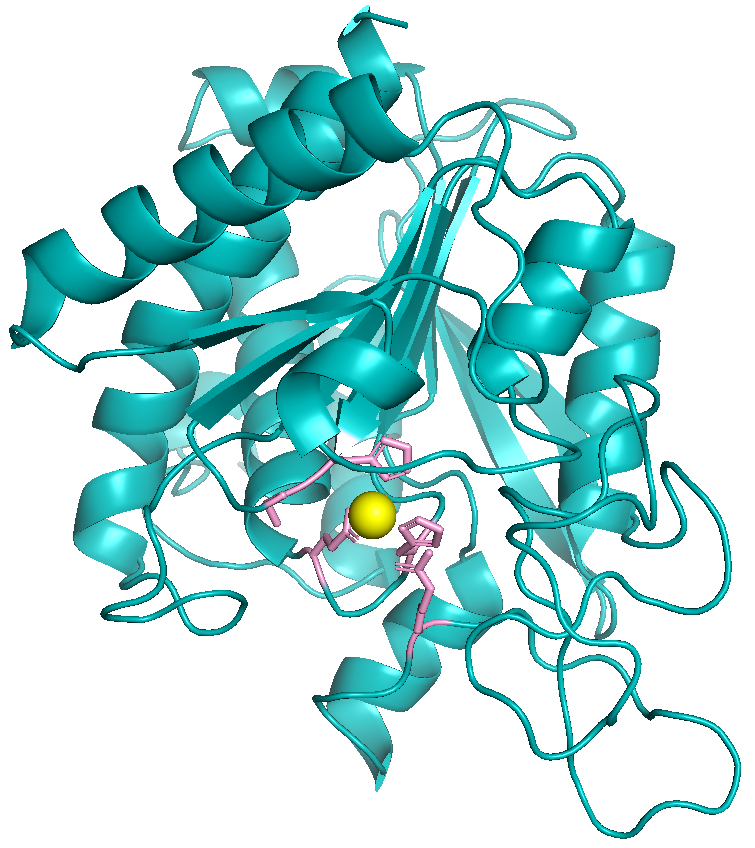}
%\caption{fig1}
\end{minipage}%
}%
\subfigure[case 2]{
\begin{minipage}[t]{0.33\linewidth}
\centering
\includegraphics[width=3.5cm]{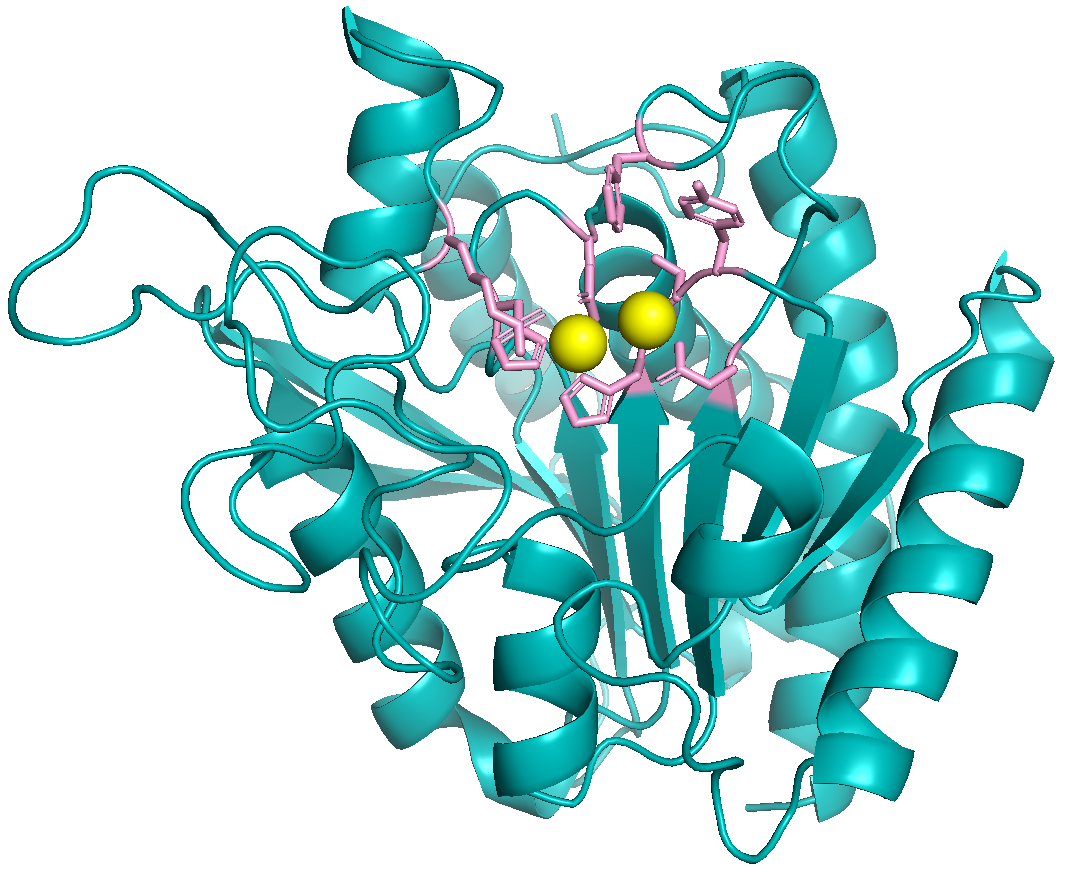}
%\caption{fig2}
\end{minipage}%
}%
\subfigure[case 3]{
\begin{minipage}[t]{0.33\linewidth}
\centering
\includegraphics[width=3.5cm]{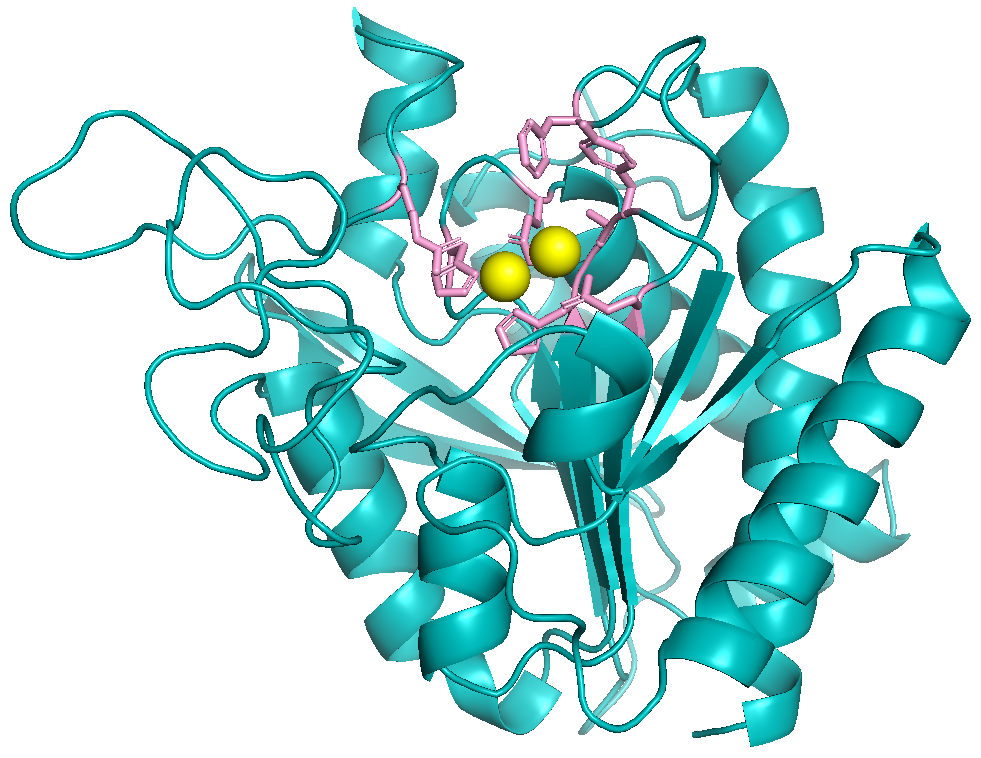}
%\caption{fig2}
\end{minipage}%
}%
\quad                
\subfigure[case 4]{
\begin{minipage}[t]{0.33\linewidth}
\centering
\includegraphics[width=3.5cm]{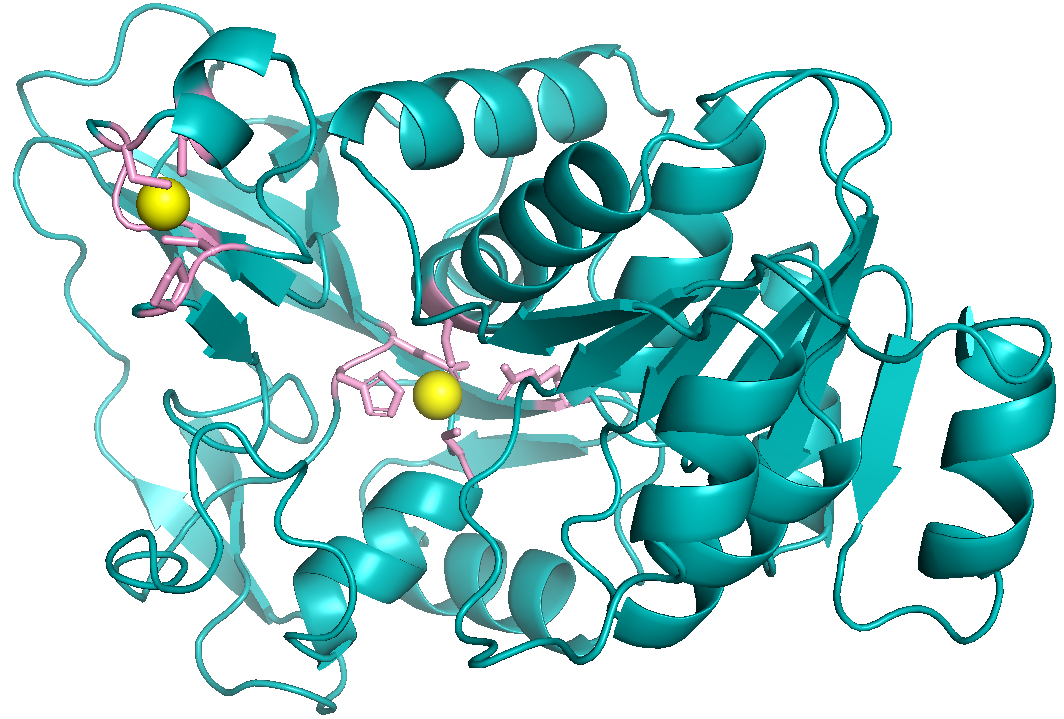}
%\caption{fig2}
\end{minipage}
}%
\subfigure[case 5]{
\begin{minipage}[t]{0.33\linewidth}
\centering
\includegraphics[width=3.5cm]{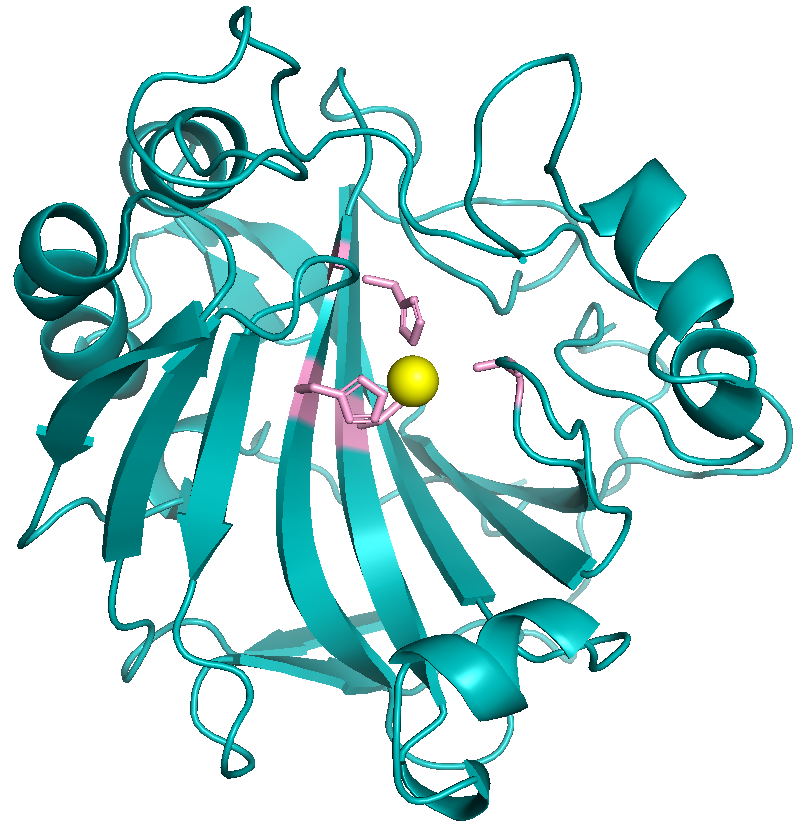}
%\caption{fig2}
\end{minipage}
}%
\subfigure[case 6]{
\begin{minipage}[t]{0.33\linewidth}
\centering
\includegraphics[width=3.5cm]{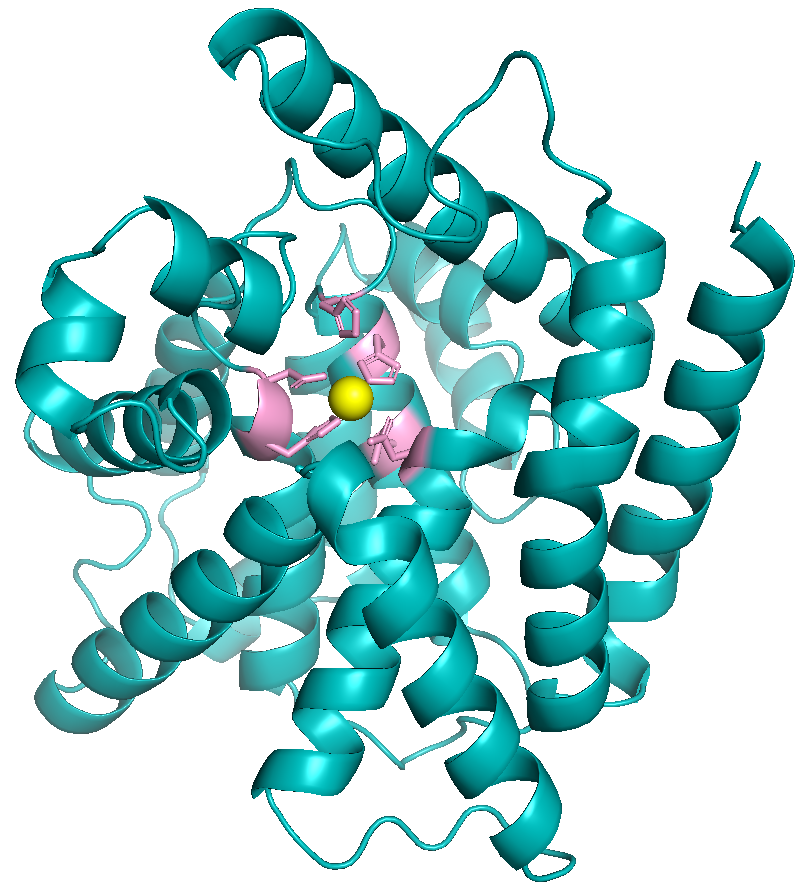}
%\caption{fig2}
\end{minipage}%
}%
\quad                 
\subfigure[case 7]{
\begin{minipage}[t]{0.33\linewidth}
\centering
\includegraphics[width=3.5cm]{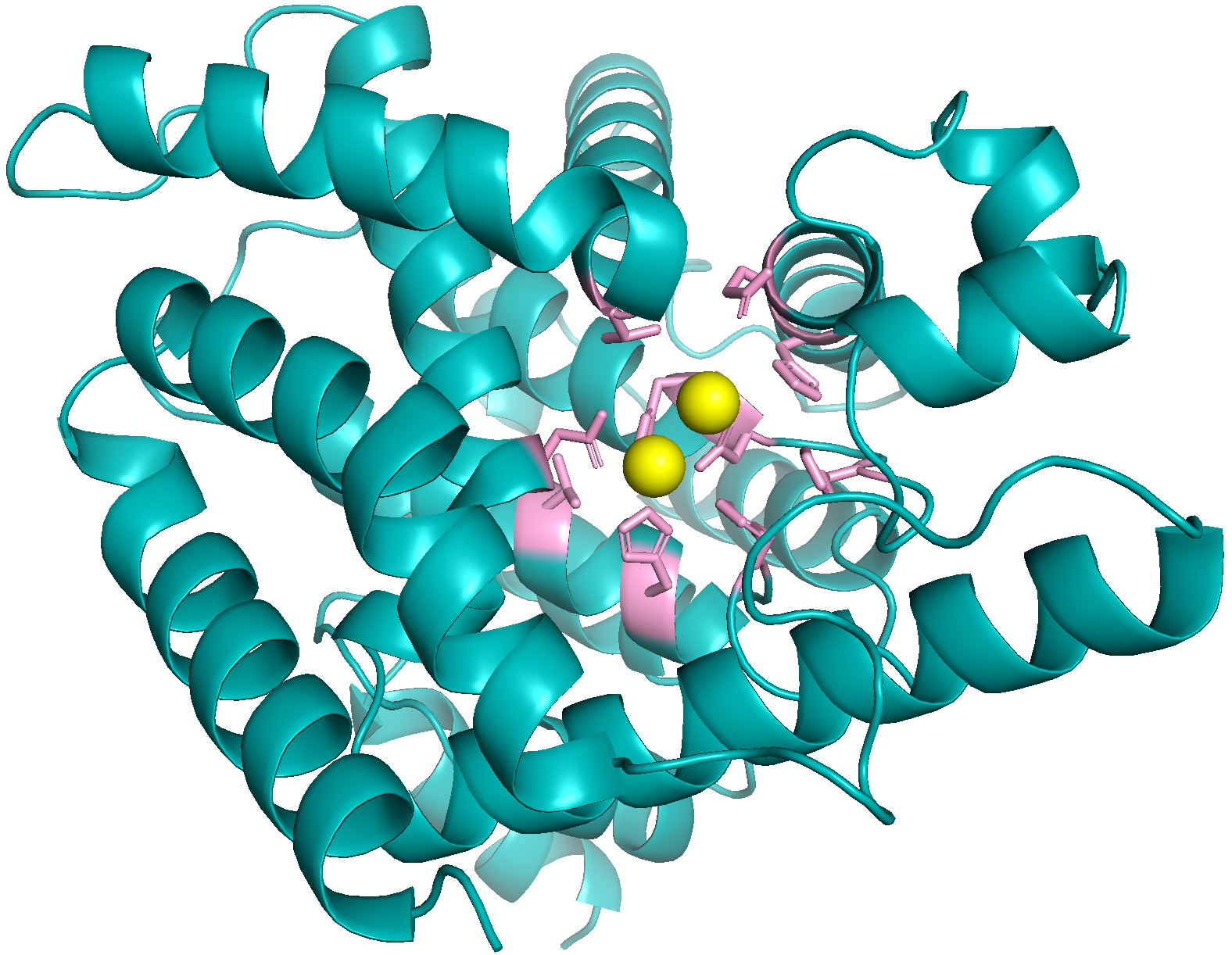}
%\caption{fig2}
\end{minipage}
}%
\subfigure[case 8]{
\begin{minipage}[t]{0.33\linewidth}
\centering
\includegraphics[width=3.5cm]{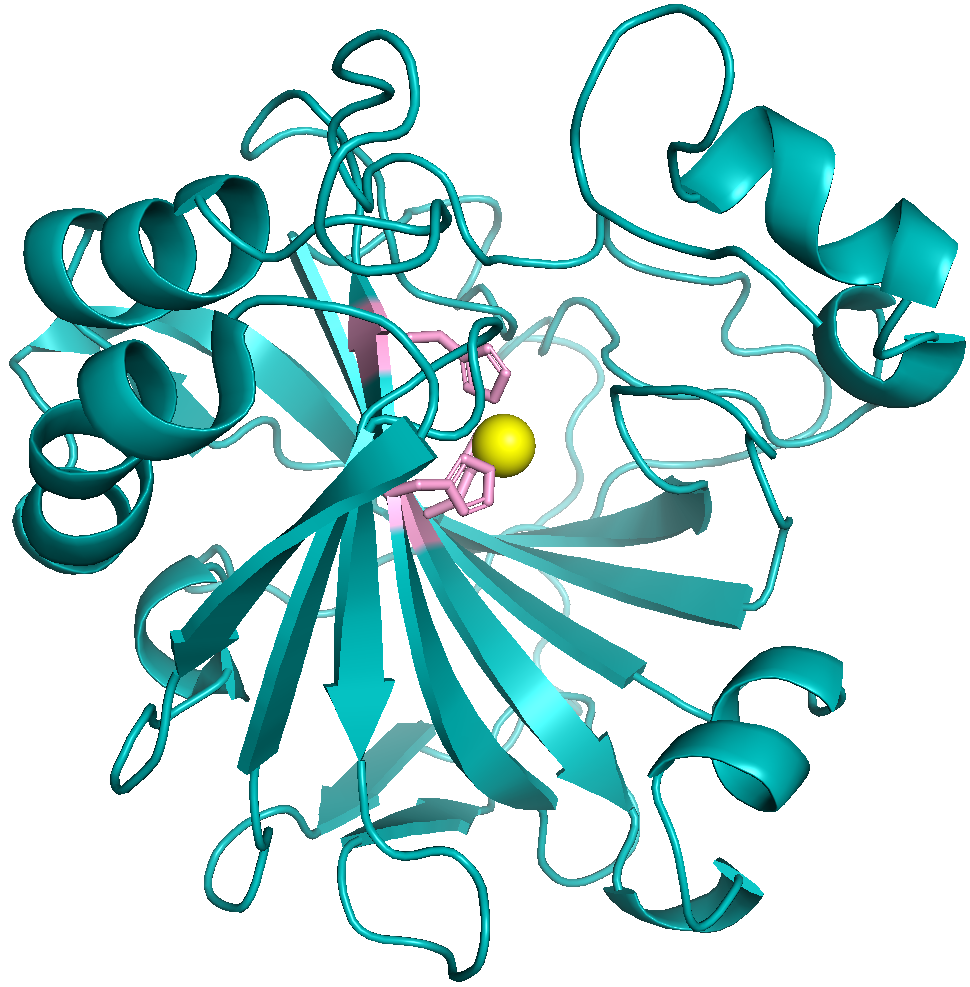}
%\caption{fig2}
\end{minipage}
}%
\subfigure[case 9]{
\begin{minipage}[t]{0.33\linewidth}
\centering
\includegraphics[width=3.5cm]{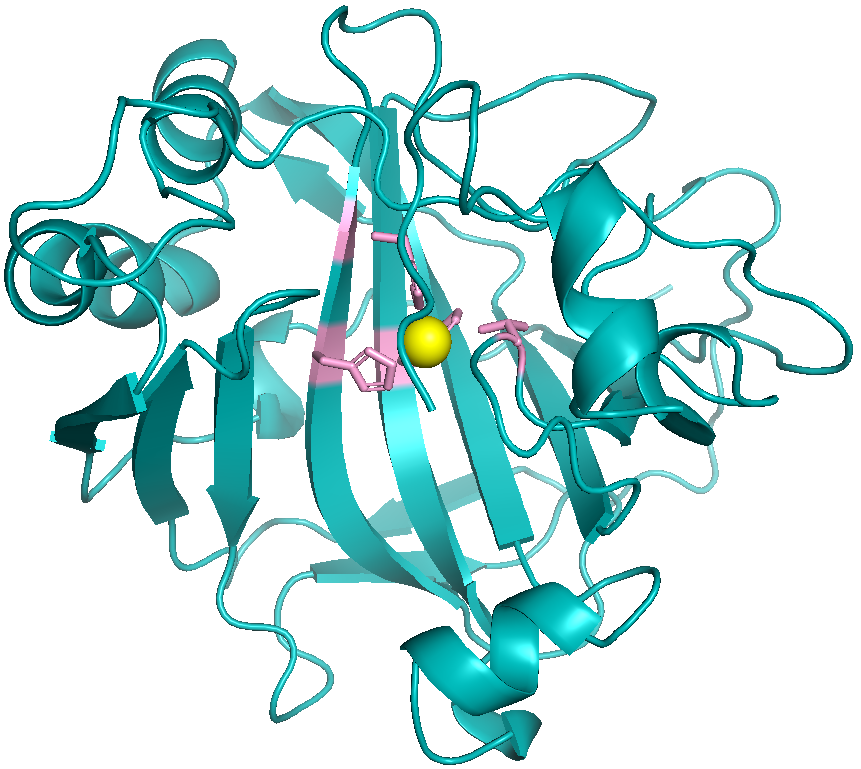}
%\caption{fig2}
\end{minipage}%
}%
\centering
\caption{More designed cases for $\beta$-lactamase.}
\label{Fig:appendix_case_beta}
\end{figure}

\begin{figure}[htbp]
\centering
\subfigure[case 1]{
\begin{minipage}[t]{0.33\linewidth}
\centering
\includegraphics[width=3.5cm]{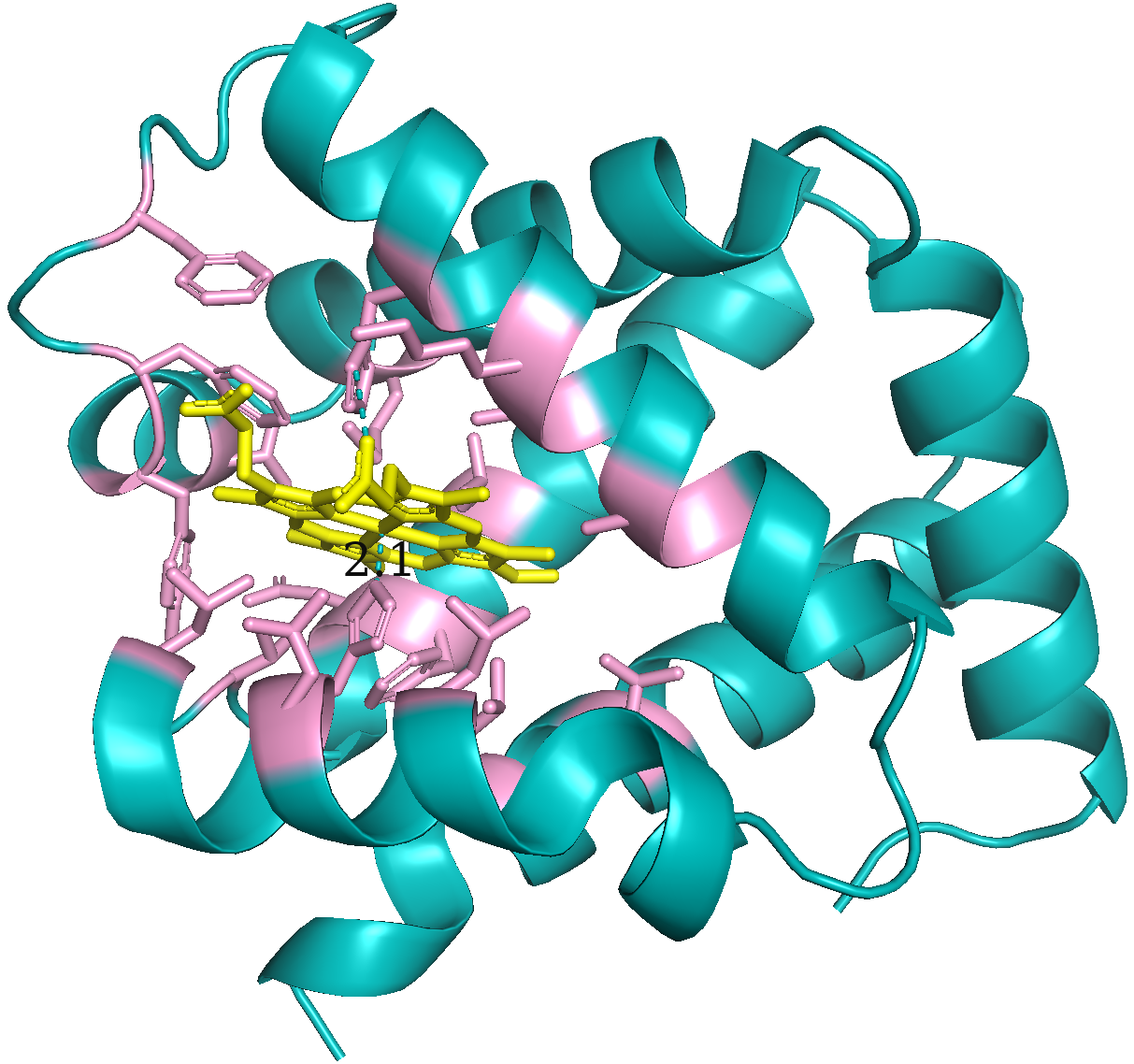}
%\caption{fig1}
\end{minipage}%
}%
\subfigure[case 2]{
\begin{minipage}[t]{0.33\linewidth}
\centering
\includegraphics[width=3.5cm]{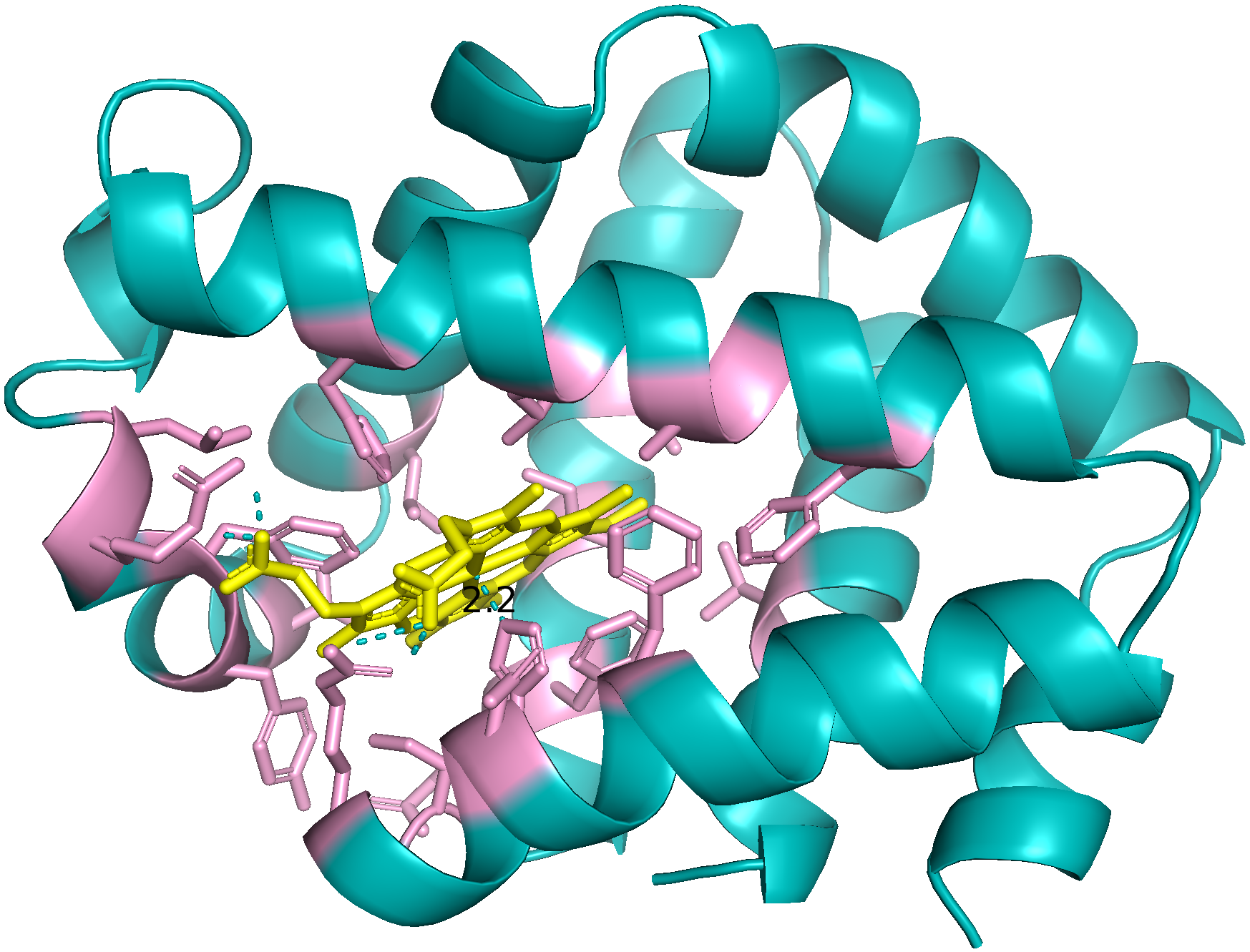}
%\caption{fig2}
\end{minipage}%
}%
\subfigure[case 3]{
\begin{minipage}[t]{0.33\linewidth}
\centering
\includegraphics[width=3.5cm]{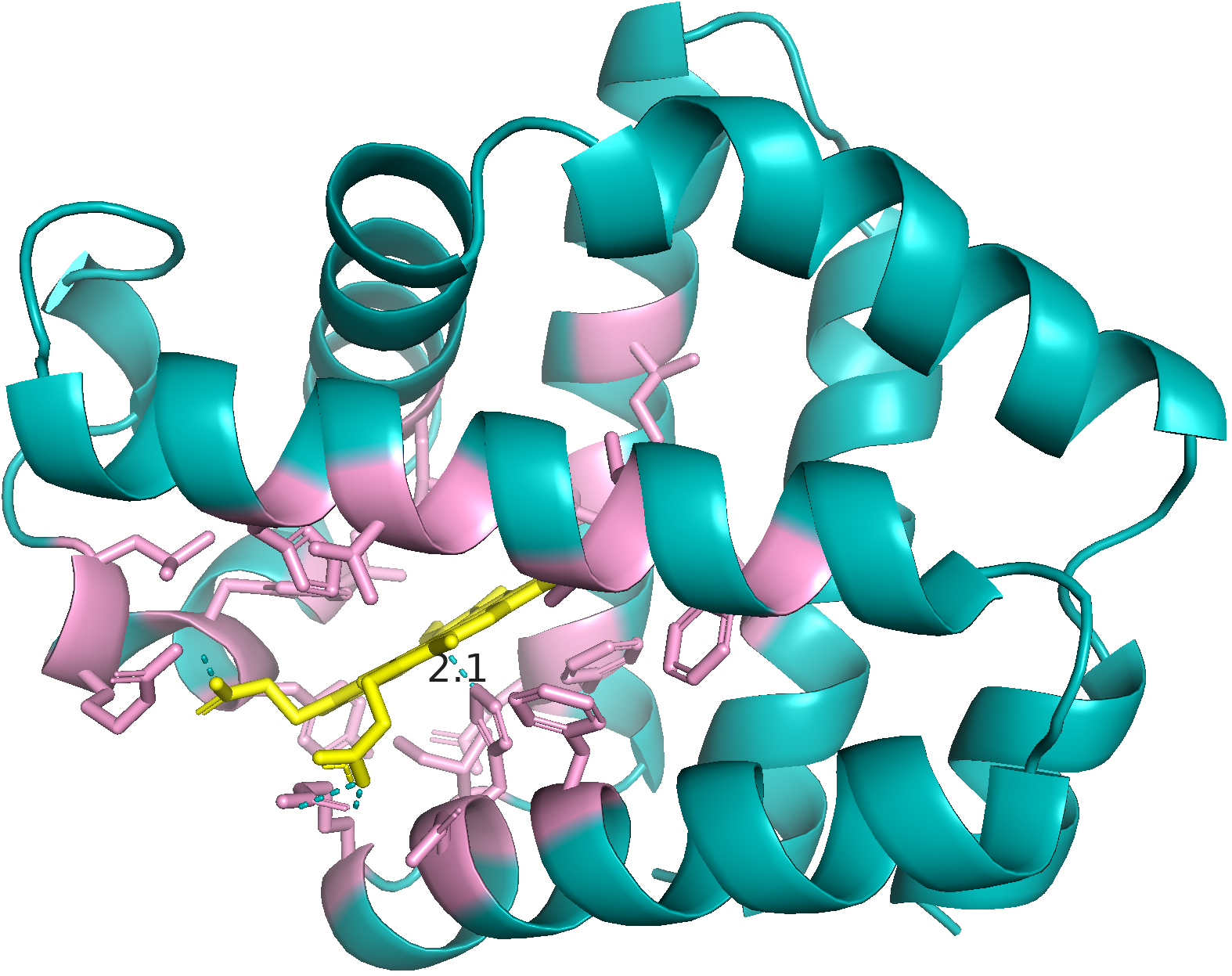}
%\caption{fig2}
\end{minipage}%
}%
\quad                 
\subfigure[case 4]{
\begin{minipage}[t]{0.33\linewidth}
\centering
\includegraphics[width=3.5cm]{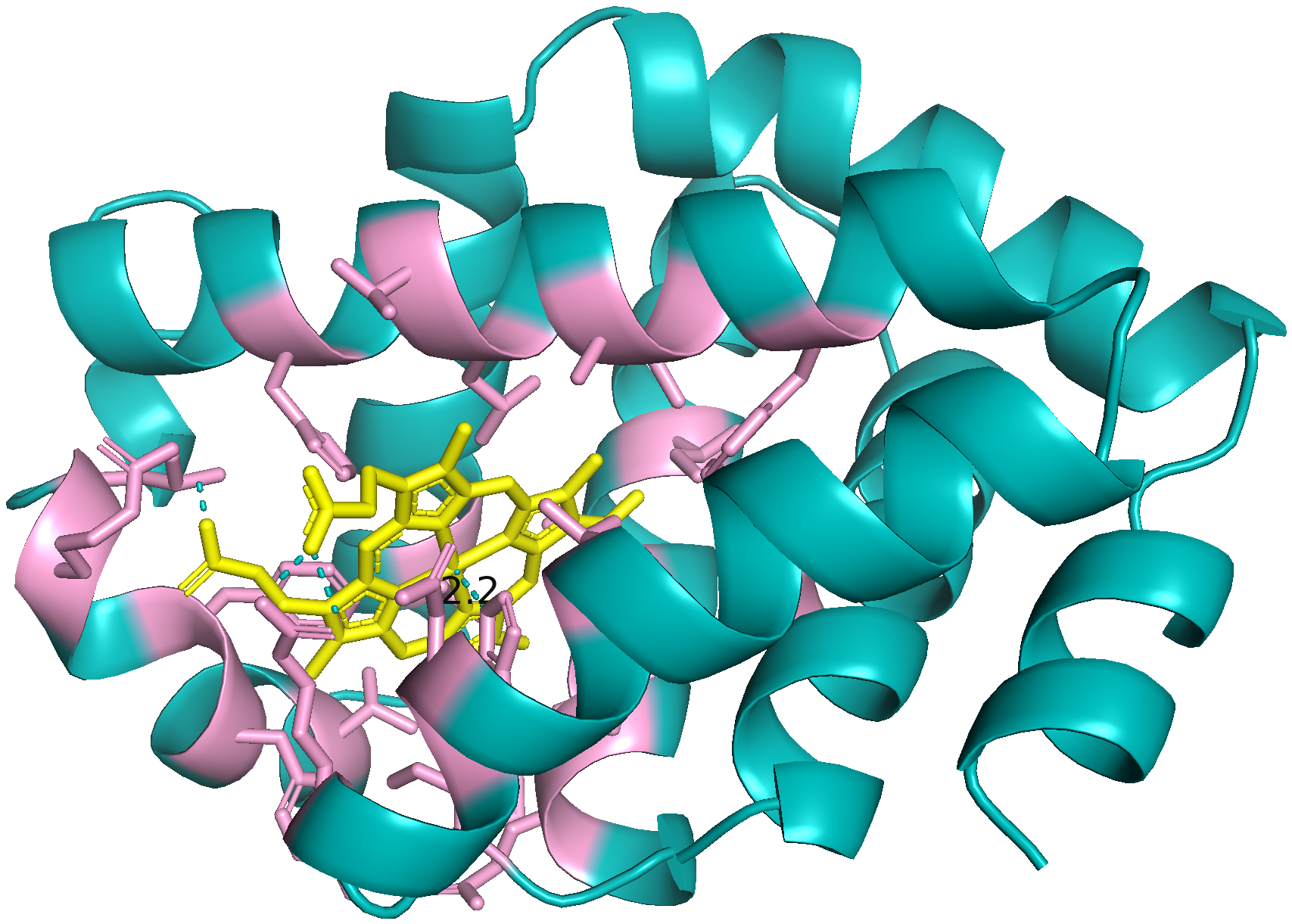}
%\caption{fig2}
\end{minipage}
}%
\subfigure[case 5]{
\begin{minipage}[t]{0.33\linewidth}
\centering
\includegraphics[width=3.5cm]{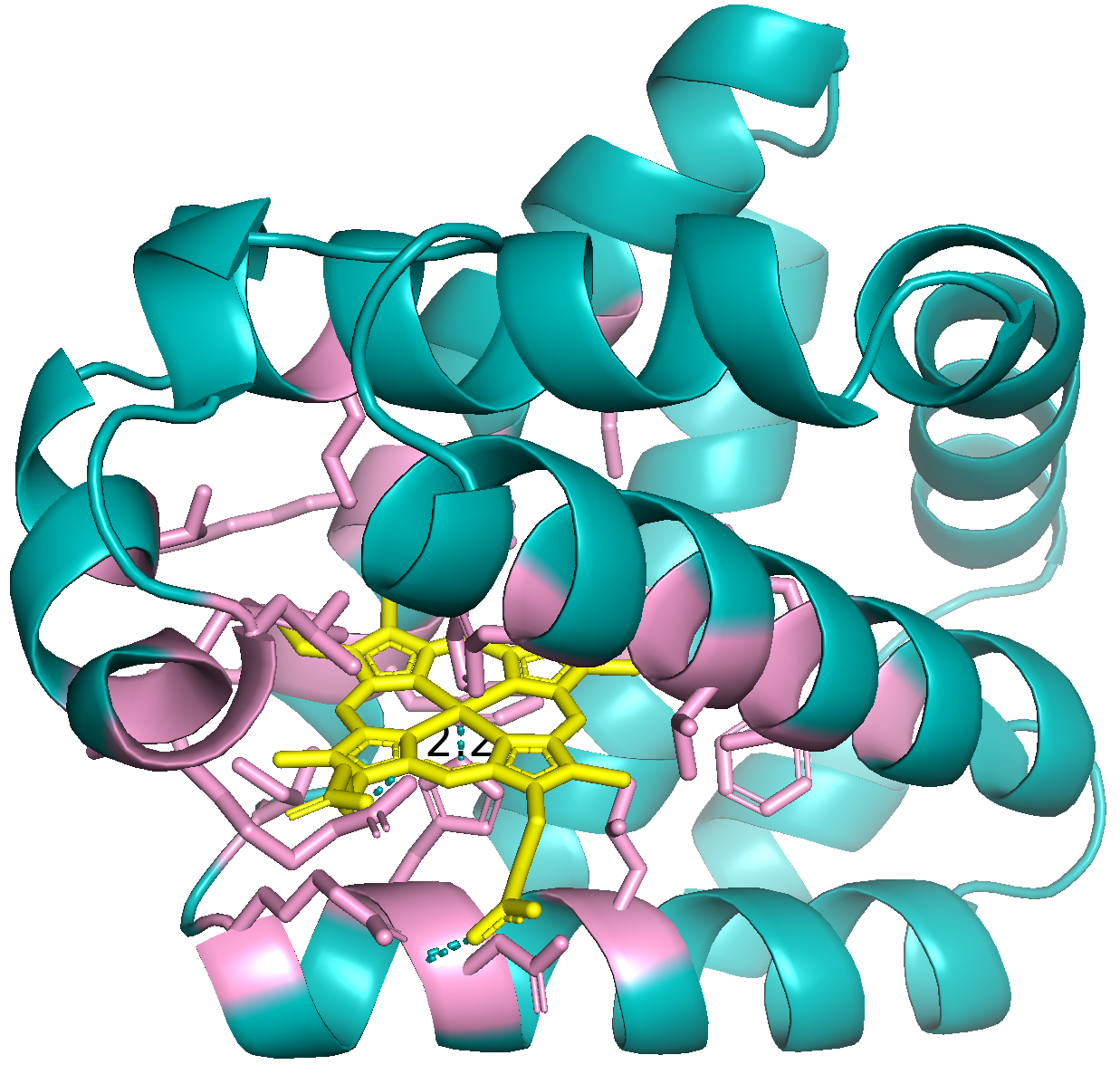}
%\caption{fig2}
\end{minipage}
}%
\subfigure[case 6]{
\begin{minipage}[t]{0.33\linewidth}
\centering
\includegraphics[width=3.5cm]{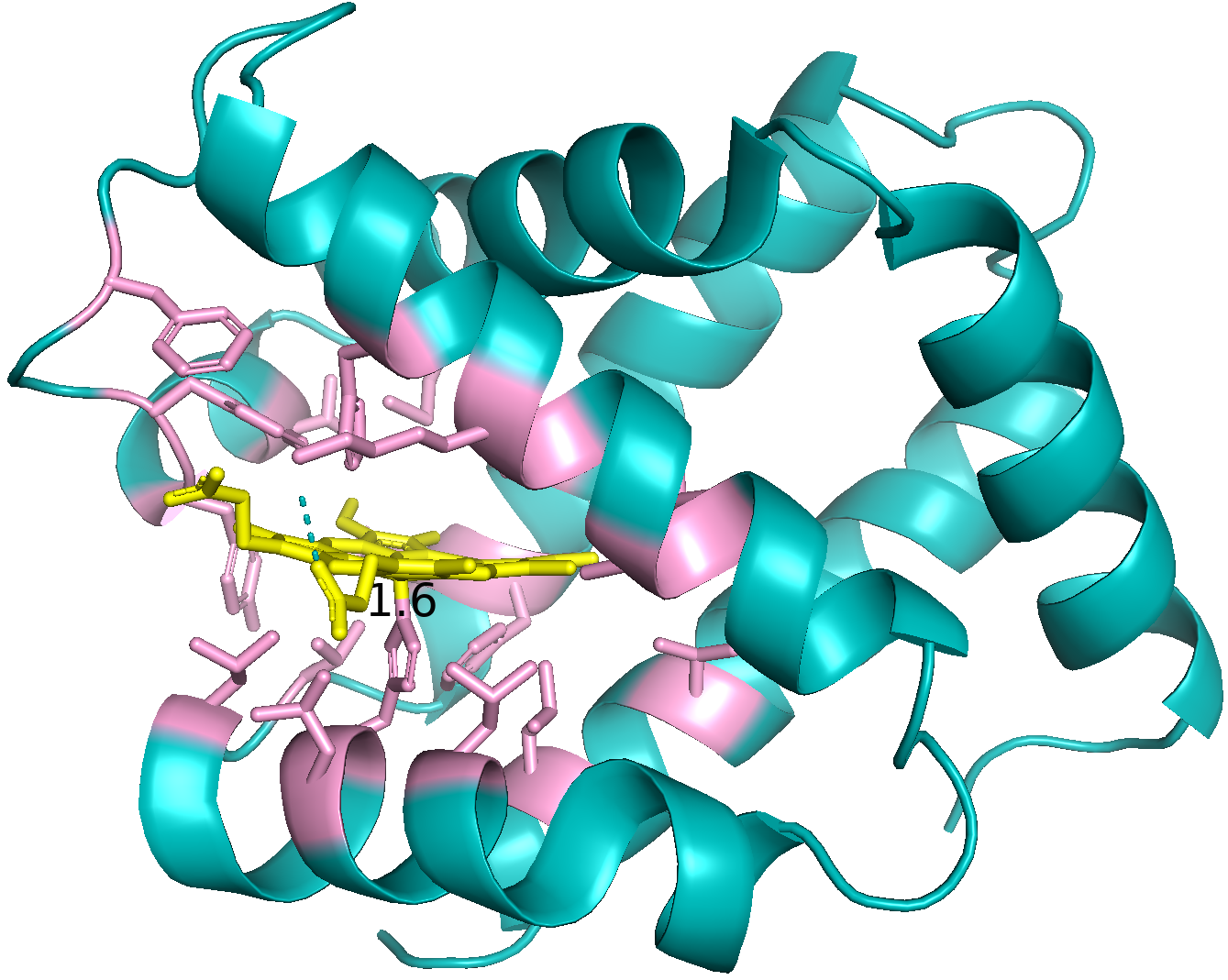}
%\caption{fig2}
\end{minipage}%
}%
\quad                 
\subfigure[case 7]{
\begin{minipage}[t]{0.33\linewidth}
\centering
\includegraphics[width=3.5cm]{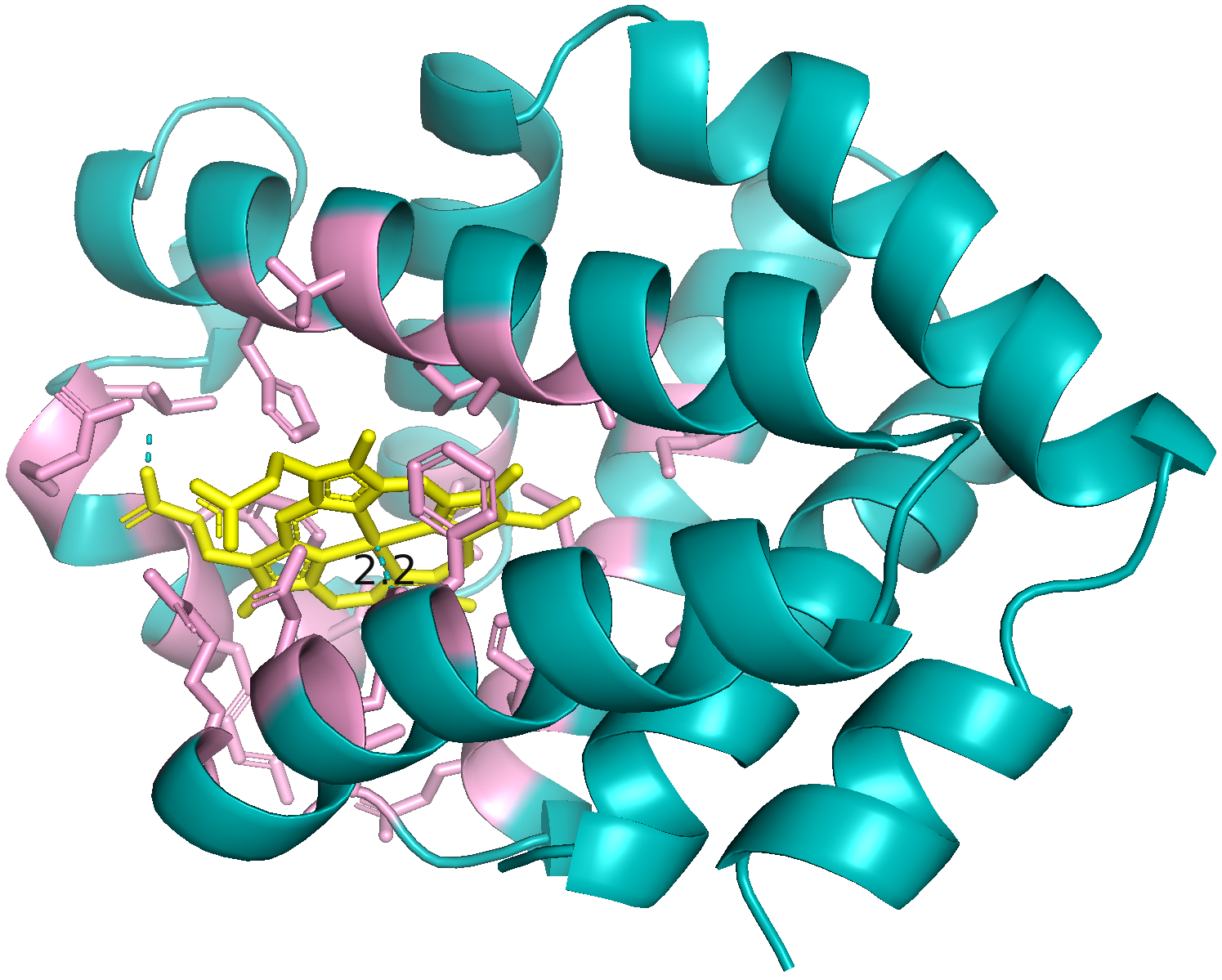}
%\caption{fig2}
\end{minipage}
}%
\subfigure[case 8]{
\begin{minipage}[t]{0.33\linewidth}
\centering
\includegraphics[width=3.5cm]{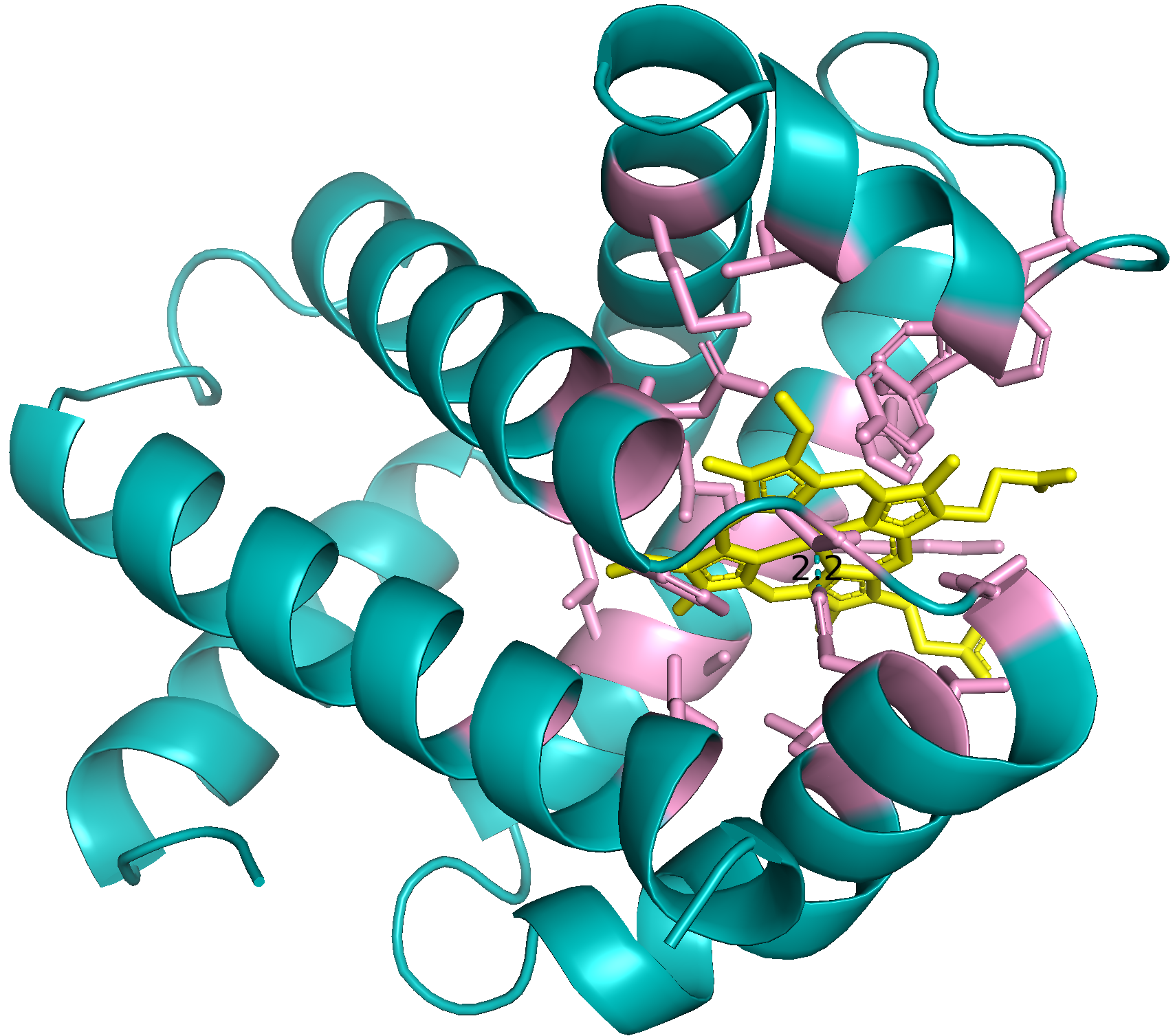}
%\caption{fig2}
\end{minipage}
}%
\subfigure[case 9]{
\begin{minipage}[t]{0.33\linewidth}
\centering
\includegraphics[width=3.5cm]{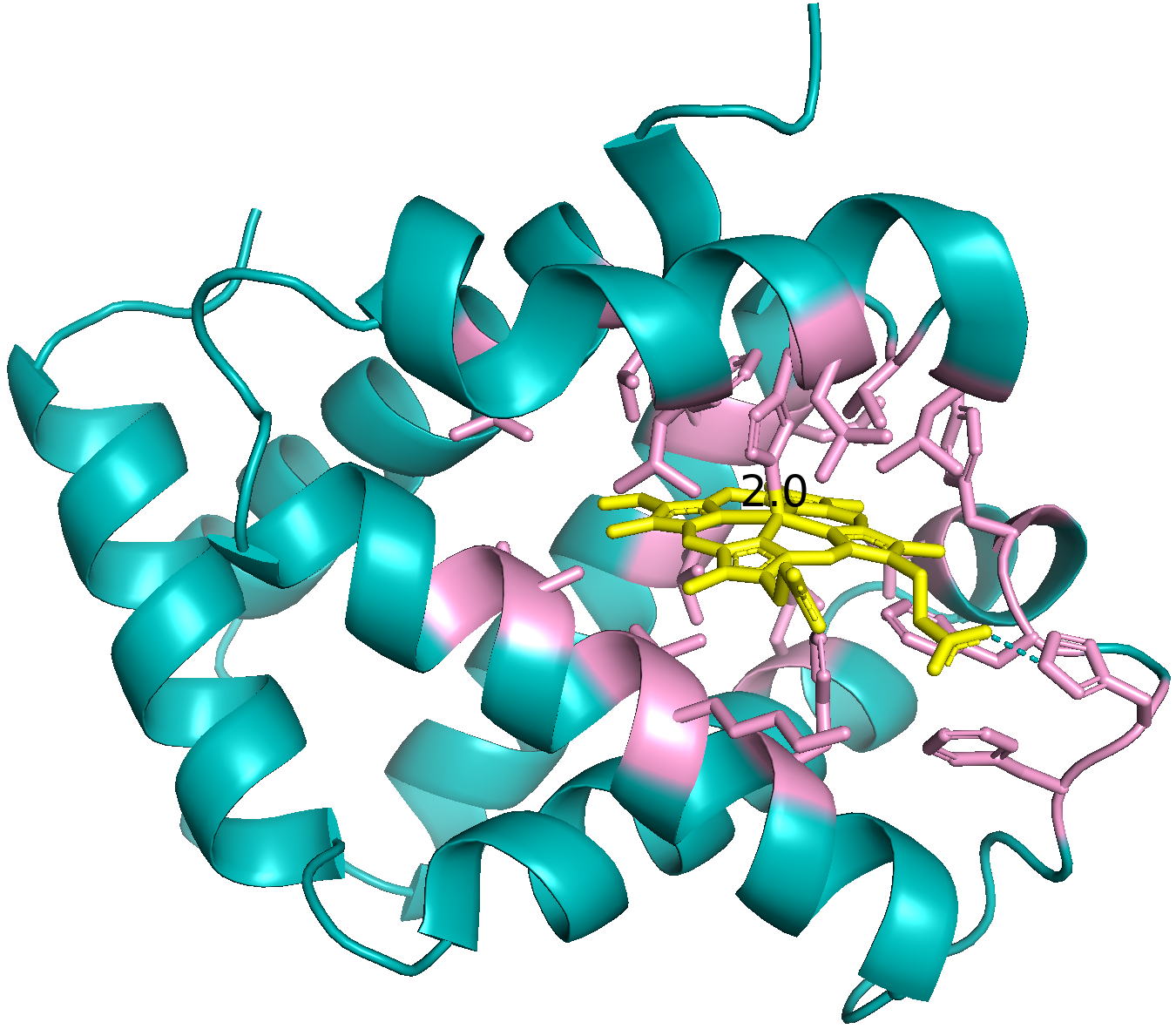}
%\caption{fig2}
\end{minipage}%
}%
\centering
\caption{More designed cases for myoglobin.}
\label{Fig:appendix_case_myoglobin}
\end{figure}

\section{More Designed Cases}
\label{appendix_cases}
Figure~\ref{Fig:appendix_case_beta} and \ref{Fig:appendix_case_myoglobin} respectively illustrate more designed proteins for $\beta$-lactamase and myoglobin. 
It shows all proteins exhibit active site environments reminiscent of those of natural proteins and can bind corresponding metallofactors, i.e., $\beta$-lactamases bind zinc ion and myoglobins bind heme, demonstrating their excellent potential to be biologically functional.
Besides, the sequences of these proteins are not present in PDB, and exhibit diverse structures, demonstrating our \model has the ability to design novel and diverse proteins with desired functions.

\end{document}